\documentclass[12pt]{article}

\usepackage{dsfont}                     
\usepackage{bm}                        

\usepackage{amsmath,amssymb,amsthm}

\theoremstyle{plain}
\newtheorem{lem}{Lemma}[section]
\newtheorem{thm}[lem]{Theorem}

\theoremstyle{definition}

\theoremstyle{remark}

\usepackage{graphicx}                  
\usepackage{epsfig}                    
\usepackage{epic,eepic}                
\usepackage{subfigure}                 
\usepackage{tikz}                      
\usetikzlibrary{shapes.geometric,arrows}
\usepackage{color}                     

\usepackage{multirow}
\usepackage{makecell}

\usepackage{algorithm}
\usepackage{algorithmic}

\usepackage[width=.95\textwidth]{caption}

\usepackage[title]{appendix}

\usepackage{ifpdf}

\newcommand{\bx}{\boldsymbol{x}}
\newcommand{\bX}{\boldsymbol{X}}
\newcommand{\by}{\boldsymbol{y}}
\newcommand{\bY}{\boldsymbol{Y}}
\newcommand{\bz}{\boldsymbol{z}}

\newcommand{\rmd}{\mathrm{d}}

\newcommand{\email}[1]{\href{mailto:#1}{\texttt{#1}}}
\def\keywords{\vspace{.5em}
	\noindent{\textbf{\textit{Keywords:}}\,}%
}
\def\endkeywords{\par}

\definecolor{darkgreen}{rgb}{0.0, 0.5, 0.0}
\usepackage[colorlinks=true,
linkcolor=darkgreen,
citecolor=darkgreen,
urlcolor=black
]{hyperref}
\usepackage{cleveref}

\title{\large\bf Coupling-Informed Transport Maps for Bayesian Filtering in Nonlinear Dynamical Systems\thanks{Updated on 2026-05-01}}
\author{Dengfei Zeng\thanks{School of Mathematical Sciences, Tongji University, Shanghai, 200092 China (\email{dfzeng@tongji.edu.cn}).}
	\and Lijian Jiang\thanks{Corresponding author. School of Mathematical Sciences, Tongji University, Shanghai, 200092 China, Key Laboratory of Intelligent Computing and Applications (Ministry of Education), Tongji University, Shanghai, 200092 China (\email{ljjiang@tongji.edu.cn})}
	\and Shuyu Sun\thanks{School of Mathematical Sciences, Tongji University, Shanghai, 200092 China, Key Laboratory of Intelligent Computing and Applications (Ministry of Education), Tongji University, Shanghai, 200092 China (\email{suns@tongji.edu.cn}, \email{xiaodunhui@tongji.edu.cn})}
	\and Dunhui Xiao \footnotemark[4]}

\date{}

\begin{document}
	\maketitle
	
	\begin{center}
		\textbf{Abstract}
	\end{center}
	\begin{quotation}
		A likelihood-free transport filtering method is proposed based on the couplings between state and observation variables. 
		By exploiting a block-triangular structure in the transport map, the analysis step of filtering is reformulated as the minimization of the maximum mean discrepancy (MMD) between the true joint measure and its transport-based approximation. 
		To circumvent the non-convexity in the MMD optimization, we introduce a training-free transport filter method via gradient flows, which leads to an analytic computation for the transport map that implies the steepest descent direction of the MMD. 
		The proposed approach accurately approximates non-Gaussian filtering posteriors and avoids particle collapse. 
		We provide a convergence analysis for the expectation of the MMD between the approximated posterior and the truth posterior. Finally, we extend the method to high-dimensional problems through domain localization.
		Numerical examples demonstrate the superior performance of our approach over conventional filtering methods in nonlinear, non-Gaussian scenarios.
		
	\end{quotation}
	
	
	\section{Introduction}
	\label{sec:1}
	In recent years, there has been growing interest in the fusion of data and physical models. Data assimilation is a powerful approach for integrating observation data with model predictions to produce reliable state estimates \cite{aschDataAssimilationMethods2016, lawDataAssimilationMathematical2015}. It has been widely used in various fields such as numerical weather prediction, signal processing, data driven modelling, and finance \cite{carrassiDataAssimilationGeosciences2018, liuPerronFrobeniusOperator2024,rothEnsembleKalmanFilter2017, wangModelFreeData2025}. Sequential Bayesian filtering constitutes a fundamental class of methods in data assimilation. It utilizes model information as the prior and incorporates the likelihood function through observation data, aiming to best combine model and data from a Bayesian posterior perspective. However, the implementation of filtering algorithms is significantly hampered by the high dimensionality and inherent nonlinearity of the underlying models, as well as the spatiotemporal scarcity of observational data.
	
	Sampling techniques provide a fundamental framework for iteratively estimating Bayesian posteriors.
	The particle filter (PF), as a class of Monte Carlo methods based on sequential importance sampling, offers a nonparametric way to approximate the posterior distribution \cite{gordonNovelApproachNonlinear1993}. 
	In the asymptotic regime, the posterior approximated by PF converges to the true posterior of the filtering problem as the number of particles increases to infinity. 
	However, PF suffers from the curse of dimensionality in high-dimensional settings, which often manifests as particle collapse during sequential updating \cite{snyderObstaclesHighdimensionalParticle2008}. 
	Localization-based particle filter methods have been proposed to mitigate the challenges of particle filters in high-dimensional problems \cite{pennyLocalParticleFilter2016, poterjoyLocalizedParticleFilter2016}. 
	These methods restrict the resampling step to multiple low-dimensional local state blocks and then combine the results to obtain a global state update. 
	However, such patching strategies may lead to non-physical discontinuities in the updated particles at the boundaries between local regions \cite{farchiReviewArticleComparison2018}. 
	Alternative data assimilation methods obtain approximate posterior estimates by imposing specific constraints on the filtering problem. 
	The Kalman Filter (KF) derives an analytical posterior under the assumption of linearity and Gaussianity, while the Extended Kalman Filter (ExKF) extends this framework to nonlinear systems through linearization using tangent linear models. 
	The Ensemble Kalman Filter (EnKF) is currently one of the most prevalent methods for nonlinear filtering and is widely used in major numerical weather prediction systems. 
	EnKF represents the Bayesian prior and posterior using an ensemble of samples and provides a reliable estimate of the state in weakly nonlinear systems through a linear transformation \cite{evensenEnsembleKalmanFilter2003, evensenSamplingStrategiesSquare2004, evensenEnsembleKalmanFilter2009}. 
	For high-dimensional systems, EnKF employs localization and covariance inflation to enhance computational efficiency and provide more accurate variance estimates \cite{huntEfficientDataAssimilation2007}. 
	However, for strongly nonlinear systems, the corresponding Bayesian prior and posterior distributions often exhibit non-Gaussian characteristics.
	The EnKF applies a uniform Kalman gain to all ensemble members, which negatively impacts on its effectiveness in such cases and leads to poor performance in highly nonlinear regimes \cite{sakovIterativeEnKFStrongly2012}. 
	A more flexible nonlinear transformation should therefore be adopted to better accommodate the complexities of strongly nonlinear systems. 
	
	The idea of transforming probability measures via (optimal) coupling/transport has a long history and wide applications in statistics, machine learning, and related fields \cite{peyreComputationalOptimalTransport2019, villaniOptimalTransportOld2009}. 
	Recently, filtering algorithms based on (optimal) coupling or transport have attracted growing attention. Given samples $\bx_1, \bx_2, \cdots, \bx_N$ from a prior distribution $p(\bx)$ and observation $\by^{o}$, the goal is to find a transformation(transport map) $\mathbf{T}$ such that it yields samples $\mathbf{T}(\bx_1), \mathbf{T}(\bx_2), \cdots, \mathbf{T}(\bx_N)$ from the posterior distribution $p(\bx \mid \by=\by^{o})$. 
	The transport map can be obtained by explicitly constructing an optimal coupling, using particle filter outputs without resampling as the posterior reference \cite{acevedoSecondorderAccurateEnsemble2017, reichNonparametricEnsembleTransform2013}. These approaches may face the curse of dimensionality similar to the particle filter (PF) in high-dimensional problems. Alternatively, transport maps can be derived by minimizing the Maximum Mean Discrepancy (MMD) with an added variance penalty; this method can be extended to high-dimensional problems by partitioning observations into multiple lower-dimensional subsets \cite{zengEnsembleTransportFilter2026}. 
	Particle flow-based filtering methods have fundamental connections with optimal transport. These approaches often characterize the transformation from prior to posterior distributions via continuous dynamics governed by ordinary differential equations. The feedback particle filter (FPF) performs Bayesian updates by introducing a mean-field type feedback control \cite{yangContinuousdiscreteTimeFeedback2014, yangFeedbackParticleFilter2013}. The variants of FPF employ deep learning or diffusion maps to approximate the feedback gain function \cite{olmezDeepFPFGain2020, taghvaeiDiffusionMapbasedAlgorithm2020}. 
	The variational mapping particle filter constructs a specific form of gradient flow by performing gradient descent on the Kullback–Leibler divergence in a reproducing kernel Hilbert space \cite{pulidoSequentialMonteCarlo2019}. This approach analytically constructs the form of the transport map. However, it relies on a Gaussian mixture approximation of the prior and requires an analytically tractable likelihood function.
	In addition, coupling techniques have been employed to construct transport maps from the prior to the posterior distribution by approximate the joint distribution with a specific structure of transport map or an optimization target. Spantini et al., \cite{spantiniCouplingTechniquesNonlinear2022}, proposes a two-step filtering method based on a monotone triangular transport map -- the Knothe-Rosenblatt rearrangement, using a standard normal distribution as an intermediate step. Likelihood-free filter methods are contructed to push forward prior to posterior based on a block-triangular structure and the optimization of discrepancy between the truth and approximated joint distribution \cite{al-jarrahNonlinearFilteringBrenier2025, baptistaConditionalSamplingMonotone2024}. 
	
	In this paper, we incoporate the idea of coupling to formulate a likelihood-free filter method. Similar to the work of Al-Jarrah et al., \cite{al-jarrahNonlinearFilteringBrenier2025}, a block-triangular transport map $\mathcal{T}$ is constructed to match two couplings: the reference joint measure $\pi_{\bx, \by}=\pi_{\by\mid\bx}\otimes\pi_{\by}$ and the approximated joint measure $\mathcal{T}_{\sharp}(\pi_{\bx}\otimes\pi_{\by})$. 
	The main contributions of the paper are twofold. 
	First, the MMD is introduced as the objective function to quantify the discrepancy between the approximate joint measure and the reference measure. It is shown that when the MMD vanishes, the state block of the block-triangular transport map pushes the prior distribution to the truth posterior. By leveraging reproducing kernel Hilbert spaces (RKHS), the kernel-based MMD enables the transport map to be trained efficiently; 
	Second, gradient flows are introduced to characterize the block-triangular transport map. Within this framework, we analytically derive the G\^{a}teaux derivative of the kernel-based MMD to obtain its steepest descent direction, and based on this, we construct a sequence of transport maps that progressively push the prior toward the posterior, without requiring explicit parametrization or training of the transport map. In this paper, the proposed method admits an upper bound of $\mathcal{O}(1/\sqrt{N})$ on the expectation of MMD between the approximate posterior and the truth posterior. We also utilize a domain localization strategy for constructing transport maps in high-dimensional settings. By appropriately partitioning the state space into localized regions, the corresponding transport maps can be trained efficiently and in parallel across different sub-regions. 
	
	The rest of the paper is organized as follows: In \Cref{sec:2}, we briefly review the sequential Bayesian filtering, the transport map for filtering, and the definition of MMD. In \Cref{sec:3}, we discuss the coupling-based transport map filter and its variant based on gradient flow. Moreover, a domain localization strategy for high-dimensional problems is presented. Several numerical examples are demonstrated to show the performance of the proposed method in \Cref{sec:4}. Finally, some conclusions were made in \Cref{sec:5}.
	
	\section{Preliminaries}
	\label{sec:2}
	
	Consider a continuous-time dynamical system governed by the stochastic differential equation
	\begin{equation}\label{eq:2.1}
		\rmd \bx_t = \mathbf{M}(\bx_t) \rmd_t + \rmd \mathbf{N}_t,
	\end{equation}
	where $\bx_t$ is a n-dimensional vector representing the state of the system at a given time and the system noise $\mathbf{N}_t$ is a $\mathbb{R}^n$-valued stochastic process. To employ discrete-time filtering methods \cite{carrassiDataAssimilationGeosciences2018}, the system is typically discretized. Applying the Euler–Maruyama scheme \cite{higham.AlgorithmicIntroductionNumerical2001} with a uniform time step $\Delta t$ yields
	\begin{equation*}
		\bx_{k\Delta t} = \bx_{(k-1)\Delta t}+ \mathbf{M}(\bx_{(k-1)\Delta t})\Delta t + \Delta \mathbf{N}_t
	\end{equation*}
	where $\Delta \mathbf{N}_t = \mathbf{N}_{k\Delta t}-\mathbf{N}_{(k-1)\Delta t}$.
	Noisy observations of the state are available at discrete times $\{t_1, t_2, \cdots, t_k, \cdots\}$, and are modeled as
	\begin{equation}\label{eq:2.2}
		\by_{t_k} = \mathbf{H}(\bx_{t_k}) + \boldsymbol{\varepsilon}_{t_k},
	\end{equation}
	where $\mathbf{H}:\mathbb{R}^n\to \mathbb{R}^m$ is the observation operator and $\boldsymbol{\varepsilon}_{t_k}$ is the observation noise at time $t_k$. 
	In the context of nonlinear stochastic dynamical systems, filtering aims to identify a state estimate that optimally aligns with both the system dynamics \eqref{eq:2.1} and the observations generated by \eqref{eq:2.2}. From a Bayesian perspective, this is achieved by utilizing posterior moments; this approach yields a point estimate while simultaneously providing a formal quantification of uncertainty.
	
	\subsection{Sequential Bayesian Filter}
	\label{sec:2.1}
	
	Given the sequence of model states $\bx_{0:K}=\{\bx_0, \bx_1, \cdots, \bx_K\}$ and observations $\by_{0:K}=\{\by_1, \by_2, \cdots, \by_K\}$, repectively. Sequential Bayesian filter iteratively estimate the posterior distribution $p_{\bx_k\mid \by_{1:k}}$ with forcast step and analysis step. 
	
	Forecast step incorporates the information of dynamical system by estimating the prior distribution at time $k$ with Chapman-Kolmogorov equation
	\begin{equation}\label{eq:2.3}
		p_{\bx_k\mid \by_{1:k-1}} = \int p_{\bx_k\mid \bx_{k-1}} p_{\bx_{k-1}\mid \by_{1:k-1}}\rmd \bx_{k-1},
	\end{equation}
	where $p_{\bx_k\mid \bx_{k-1}}$ is the transition probability govern by dynamical system \eqref{eq:2.1}. Analysis step fits the observations via the Bayes' rule
	\begin{equation}\label{eq:2.4}
		p_{\bx_k\mid \by_{1:k}} = \frac{p_{\bx_{k}\mid \by_{1:k-1}}p_{\by_k\mid \bx_k}}{p_{\by_k\mid \by_{1:k-1}}},
	\end{equation}
	where $p_{\by_k\mid \bx_k}$ is the likelihood given by the observation process \eqref{eq:2.2}. Marginal likelihood $p_{\by_k\mid \by_{1:k-1}}$ is defined by $p_{\by_k\mid \by_{1:k-1}}=\int p_{\bx_k\mid \by_{1:k-1}}p_{\by_k\mid \bx_k}\rmd \bx_k$. 
	
	For Linear-Gaussian cases, the forecast and analysis step can be computed analytically, which gives the Kalman filter. 
	To accommodate nonlinear dynamics, the ExKF introduces a tangent linear approximation, formally represented in \eqref{eq:2.1}–\eqref{eq:2.2}, yet it remains prone to failure in strongly nonlinear regimes.
	Particle filters offer an asymptotically exact posterior approximation for nonlinear problems but are hampered by severe particle degeneracy in high-dimensional spaces.
	Ensemble-based filters use an ensemble of particles to characterize the uncertainty of states and observations. 
	Specifically, the EnKF applies a linear transformation to update the ensemble, effectively adjusting the mean and covariance. To enhance flexibility, nonlinear transformations may be employed for more robust posterior approximations. In this work, we propose constructing a transformation from the prior to the posterior ensemble using transport maps
	
	\subsection{Transport Maps for Filtering}
	\label{sec:2.2}
	
	In the forecast step, starting from an ensemble $\{\bx_{k-1}^{(i)}\}_{i=1}^N$ drawn from the previous posterior distribution $p_{\bx_{k-1}\mid \by_{1:k-1}}$, the samples $\{\hat{\bx}_{k}^{(i)}\}_{i=1}^N$ of current prior distribution $p_{\bx_k\mid \by_{1:k-1}}$ are obtained by numerical simulation of the dynamical system \eqref{eq:2.1}. The simulation process can be formulated as a transport map $\mathcal{M}(\cdot)=\text{id}(\cdot)+\mathbf{M}(\cdot)\Delta t + \Delta \mathbf{N}_t$, it gives
	\begin{equation}\label{eq:2.5}
		\hat{\bx}_{k} = \mathcal{M}(\bx_{k-1}) \Leftrightarrow \pi_{\bx_k\mid \by_{1:k-1}} = \mathcal{M}_{\sharp} \pi_{\bx_{k-1}\mid \by_{1:k-1}},
	\end{equation}
	where operator $\mathcal{M}_{\sharp}$ denotes the push-forward operator, $\pi_{\bx_k\mid \by_{1:k-1}}$ and  $\pi_{\bx_{k-1}\mid \by_{1:k-1}}$ denote the probability measures associated with the densities $p_{\bx_k\mid \by_{1:k-1}}$ and $p_{\bx_{k-1}\mid \by_{1:k-1}}$, respectively. 
	
	Similarly, a transport map $\mathcal{T}$ can be designed to represent the analysis step. The associated push-forward operator $\mathcal{T}_{\sharp}$ satisfies 
	\begin{equation}\label{eq:2.6}
		\pi_{\bx_k\mid \by_{1:k}} = \mathcal{T}_{\sharp} \pi_{\bx_k\mid \by_{1:k-1}}. 
	\end{equation}
	Note that the transport map $\mathcal{T}$ depends explicitly on the current observation $\by_k$. Consequently, posterior samples can be generated by simply applying the map $\mathcal{T}$ to the prior ensemble $\{\hat{\bx}_k^{(i)}\}_{i=1}^N$. A schematic overview of the transport-based filter is provided in \Cref{fig:filter}.
	
	\begin{figure}[!htbp]
		\centering
		\includegraphics{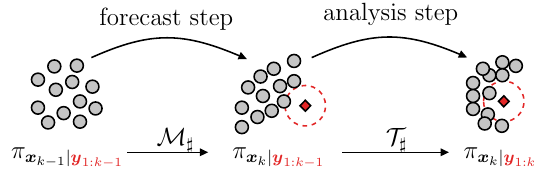}
		\caption{Propagation of particles in transport-based filtering}
		\label{fig:filter}
	\end{figure}
	
	Equation \eqref{eq:2.6} specifies the governing conditions that an admissible transport map $\mathcal{T}$ must be satisfied. To ensure that the transport measure $\mathcal{T}_{\sharp} \pi_{\bx_k\mid \by_{1:k-1}}$ is as close as possible to the truth posterior measure $\pi_{\bx_k\mid \by_{1:k}}$, a metric $d(\pi_1, \pi_2)$ of the discrepancy between two probability measures is required. The analysis step of filtering can be then written as an optimization problem
	\begin{equation}\label{eq:2.7}
		\mathcal{T} = \underset{\mathcal{T}}{\min }\ d(\pi_{\bx_k\mid \by_{1:k}}, \mathcal{T}_{\sharp} \pi_{\bx_k\mid \by_{1:k-1}}).
	\end{equation}
	
	The transport map  $\mathcal{T}$  can be formulated as the flow of ODEs or the composition of simple functions, such as neural networks. The discrepancy $d(\pi_1, \pi_2)$ between two measures is typically quantified using metrics like the KL divergence or the Wasserstein distance. In this paper, the distance metric $d$ in \eqref{eq:2.7} is specified as the Maximum Mean Discrepancy. A key advantage of MMD is able to measure the discrepancy between two distributions using only their empirical samples, without invoking any parametric assumptions or explicit knowledge of the underlying densities.
	
	\subsection{Maximum mean discrepancy}
	\label{sec:2.3}
	
	The MMD is a non-parametric statistical measure to quantify the difference between two probability distributions $p$ and $q$. Based on the definition of weak convergence of probability measures, we have the following lemma.
	\begin{lem}[Lemma 1 of \cite{grettonKernelMethodTwoSampleProblem2006}]
		Let $\Omega$ be a metric space, $p$ and $q$ be two probability measure defined on $\Omega$. Then $p=q$ if and only if $\mathbb{E}_{\bx \sim p}[f(\bx)] = \mathbb{E}_{\by \sim q}[f(\by)]$ for all $f\in \mathcal{C}(\Omega)$, where $\mathcal{C}(\Omega)$ is the space of bounded continuous function on $\Omega$. 
	\end{lem}
	MMD is defined as the largest possible difference between these expectations over the function class $\mathcal{F}$:
	$$
	\text{MMD}[p, q; \mathcal{F}] := \sup_{f \in \mathcal{F}} \left( \mathbb{E}_{\bx \sim p}[f(\bx)] - \mathbb{E}_{\by \sim q}[f(\by)] \right).
	$$
	When $\mathcal{F}$ is expressive enough (e.g., the space of all bounded continuous functions $\mathcal{C}_b$), MMD equals to zero if and only if $p = q$, making it a well-defined divergence measure.
	
	In practice, $\mathcal{F}$ is typically chosen as the unit ball in a RKHS $\mathcal{H}$ induced by a positive definite kernel $k: \Omega \times \Omega \to \mathbb{R}$:  
	$$
	\mathcal{F} = \{ f \in \mathcal{H} : \|f\|_{\mathcal{H}} \leq 1 \}.
	$$  
	Under this setting, the kernel MMD admits a closed-form expression:
	$$
	\begin{aligned}
		\text{MMD}^2[p, q; \mathcal{F}] &= \|\mathbb{E}_{p(\bx)}\left[k(\bx, \cdot)\right]-\mathbb{E}_{q(\by)}\left[k(\by, \cdot)\right]\|^2_{\mathcal{H}}\\
		&=\mathbb{E}_{\bx, \bx^{\prime} \sim p}[k(\bx, \bx^{\prime})] - 2\,\mathbb{E}_{\bx \sim p, \by \sim q}[k(\bx, \by)] + \mathbb{E}_{\by, \by^{\prime} \sim q}[k(\by, \by^{\prime})].
	\end{aligned}
	$$
	Given independent samples $\{\bx_i\}_{i=1}^N \sim p$ and $\{\by_j\}_{j=1}^N \sim q$, the empirical approximation of $\text{MMD}^2$ is:
	$$
	\widehat{\text{MMD}}^2[p, q; \mathcal{F}] = \frac{1}{N^2} \sum_{i=1}^{N}\sum_{j=1}^{N} k(\bx_i, \bx_{j}) - \frac{2}{N^2} \sum_{i=1}^{N}\sum_{j=1}^{N} k(\bx_i, \by_j) + \frac{1}{N^2} \sum_{i=1}^{N}\sum_{j=1}^{N} k(\by_i, \by_{j}).
	$$
	When using the MMD as the loss function, one can further enhance training performance by incorporating additional penalty terms, such as the maximum covariance discrepancy \cite{zhangMaximumMeanCovariance2020} and a variance penalty based on the reference distribution \cite{zengEnsembleTransportFilter2026}.
	
	\section{The coupling-informed transport filter}\label{sec:3}
	
	A central challenge in Bayesian filtering lies in the treatment of the normalization constant in Bayes' formula \eqref{eq:2.4} during the analysis step. We note that the joint distribution in the Bayesian probabilistic model encapsulates the complete information of the system. The prior and posterior distributions of the state variables can be conceptualized as distinct couplings. Utilizing the transport map framework, we implicitly construct a transformation from the prior to the posterior by matching their respective joint distributions.
	
	From Bayes rule, the joint distribution can be written as the following form,
	\begin{equation}\label{eq:3.1}
		p(\bx, \by) = p(\bx\mid \by)p(\by) = p(\by\mid \bx)p(\bx).
	\end{equation}
	Based on \eqref{eq:3.1}, we can construct three different couplings of the random variables $\bx$ and $\by$:  the posterior coupling $\pi_{\bx\mid \by}\otimes \pi_{\by}$, the likelihood coupling $\pi_{\by\mid \bx}\otimes \pi_{\bx}$, and indepednet coupling $\pi_{\bx}\otimes \pi_{\by}$.
	Simulating the system dynamics \eqref{eq:2.1} and the observation equation \eqref{eq:2.2} yields two copies of coupled variables, denoted by $(\bx, \by)$ and $(\bar{\bx}, \bar{\by})$, which give rise to the prior measure $\pi_{\bx}$ and the likelihood measure $\pi_{\by\mid \bx}$, respectively. 
	Consequently, the pair $({\bx}, \bar{\by})$, formed by taking ${\bx}$ from the pair $(\bx, \by)$ and $\bar{\by}$ from the other pair $(\bar{\bx}, \bar{\by})$, constitutes a pair of the independent coupling $\pi_{\bx}\otimes \pi_{\by}$. 
	
	The idea of this work is to construct a transport map $\mathcal{T}:\mathbb{R}^{n+m}\to \mathbb{R}^{n+m}$ that transforms independent coupling $\pi_{\bx}\otimes \pi_{\by}$ to the posterior coupling $\pi_{\bx\mid \by}\otimes \pi_{\by}$. Specifically, we constrain the transport map to be the block-triangular form,
	\begin{equation}\label{eq:3.2}
		\mathcal{T}({\bx}, \bar{\by}) = \left[
		\begin{matrix}
			\mathbf{T}({\bx}, \bar{\by})\\
			\bar{\by}
		\end{matrix}\right].
	\end{equation}
	In the construction of transport map $\mathcal{T}$, the first component $\mathbf{T}(\cdot, {\by})$ pushes forward prior measure $\pi_{\bx}$ to the posterior measure $\pi_{\bx\mid {\by}}$,  while the second component is the identity map that preserves the marginal measure $\pi_{\by}$. 
	
	Note that $\mathcal{T}_{\sharp}(\pi_{\bx}\otimes \pi_{\by})$ can be interpreted as an approximation to the posterior coupling $\pi_{\bx\mid \by}\otimes \pi_{\by}$. 
	Since both $p(\mathbf{x} \mid \mathbf{y}) p(\mathbf{y})$ and $p(\mathbf{y} \mid \mathbf{x}) p(\mathbf{x})$ represent the joint distributions $p(\mathbf{x}, \mathbf{y})$, we may use the likelihood coupling $\pi_{\by\mid \bx}\otimes \pi_{\bx}$ as the reference.
	A necessary condition for an admissible transport map $\mathcal{T}$ is that the measures $\pi^{\mathcal{T}}_1:=\mathcal{T}_{\sharp}(\pi_{\bx}\otimes \pi_{\by})$ and $\pi_2:=\pi_{\by\mid \bx}\otimes \pi_{\bx}$ must match each other. For simplicity, we denote $\mathbf{z}=(\bx, \by)$. The MMD is used to measure the dicrepancy,
	\begin{equation}\label{eq:3.3}
		\operatorname{MMD}[\pi^{\mathcal{T}}_1, \pi_2, \mathcal{F}] = \underset{f\in \mathcal{F}}{\sup}\left(\mathbb{E}_{\mathbf{z}\sim\pi^{\mathcal{T}}_1}[f(\mathbf{z})]-\mathbb{E}_{\mathbf{z}^{\prime}\sim\pi_2}[f(\mathbf{z}^{\prime})]\right),
	\end{equation}
	where the space $\mathcal{F}$ is the unit ball of the RKHS $\mathcal{H}$ with the radius basis function kernel $k(\mathbf{z}, \mathbf{z}^{\prime})=\exp\left(-{\|\mathbf{z}-\mathbf{z}^{\prime}\|^2}/{b_w^2}\right)$, $b_w$ is called the bandwidth. 
	
	\begin{thm}\label{thm:3.1}
		If the MMD with between $\mathcal{T}_{\sharp} (\pi_{\bx}\otimes \pi_{\by})$ and $\pi_{\by\mid \bx}\otimes \pi_{\bx}$ is zero and the transport map $\mathcal{T}$ is constructed by the block-triangular form as \eqref{eq:3.2}, then the component transport map $\mathbf{T}$ in $\mathcal{T}$ push forward the prior $\pi_{\bx}$ toward the posterior $\pi_{\bx\mid \by}$.
	\end{thm}
	\begin{proof}
		Since the MMD between $\mathcal{T}_{\sharp} (P(\bx)\otimes P(\by))$ and $P(\bx, \by)$ is zero, then
		\begin{equation}\label{eq:3.4}
			\mathbb{E}_{(\bar{\bx}, \bar{\by})\sim \pi_{\by\mid \bx}\otimes \pi_{\bx}}\left[f(\bar{\bx}, \bar{\by})\right] = \mathbb{E}_{({\bx}, \bar{\by})\sim \pi_{\bx}\otimes \pi_{\by}}\left[f(\mathbf{T}({\bx}, \bar{\by}), \bar{\by})\right]
		\end{equation}
		hold true for all $f\in \mathcal{H}$. According to the universal approximation property of an universal kernel, the RKHS $\mathcal{H}$ is dense in the space of continuous bounded functions $\mathcal{C}_b$. Then, Equation \eqref{eq:3.4} hold true for all $f\in \mathcal{C}_b$. Taking arbitrary continuous bounded functions $f_1(\bx)$ and $f_2(\by)$, then $f(\bx, \by)=f_1(\bx)f_2(\by) \in \mathcal{C}_b$. It gives
		\begin{equation*}
			\mathbb{E}_{(\bar{\bx}, \bar{\by})\sim \pi_{\by\mid \bx}\otimes \pi_{\bx}}\left[f_1(\bar{\bx})f_2(\bar{\by})\right] = \mathbb{E}_{({\bx}, \bar{\by})\sim \pi_{\bx}\otimes \pi_{\by}}\left[f_1(\mathbf{T}({\bx}, \bar{\by}))f_2(\bar{\by})\right].
		\end{equation*}
		From the projection definition of conditional expectation,
		\begin{equation*}
			\begin{aligned}
				\mathbb{E}\left[f_1(\bar{\bx})f_2(\bar{\by})\right] &= \mathbb{E}\left[\mathbb{E}\left[f_1(\bar{\bx})\mid \bar{\by}\right]f_2(\bar{\by})\right],\\
				\mathbb{E}\left[f_1(\mathbf{T}({\bx}, \bar{\by}))f_2(\bar{\by})\right] &= \mathbb{E}\left[\mathbb{E}\left[f_1(\mathbf{T}({\bx}, \bar{\by}))\mid \bar{\by}\right]f_2(\bar{\by})\right],
			\end{aligned}
		\end{equation*}
		it follows that
		\begin{equation*}
			\mathbb{E}\left[f_1(\bar{\bx})\mid \bar{\by}\right] = \mathbb{E}\left[f_1(\mathbf{T}({\bx}, \bar{\by}))\mid \bar{\by}\right], a.e.-\pi_{\by}.
		\end{equation*}
		Because $({\bx}, \bar{\by})\sim \pi_{\bx}\otimes \pi_{\by}$, and ${\bx}$ is independent of $\bar{\by}$, then for all $f_1\in \mathcal{C}_b$
		\begin{equation*}
			\mathbb{E}\left[f_1(\bar{\bx})\mid \bar{\by}\right] = \mathbb{E}\left[f_1(\mathbf{T}({\bx}, \bar{\by}))\right], a.e.-\pi_{\by}.
		\end{equation*}
		Thus $\mathbf{T}(\cdot, \by)_{\sharp}\pi_{\bx} = \pi_{\bx\mid \by}$.
	\end{proof}
	
	\Cref{thm:3.1} shows that, when using the block triangular structure defined in \eqref{eq:3.2} and minimizing the MMD, the component map $\mathbf{T}(\cdot, \by)$ pushes forward the prior measure $\pi_{\bx}$ to the posterior measure $\pi_{\bx\mid \by}$, that is, $\mathcal{T}$ essentially constructs the posterior coupling $\pi_{\bx\mid \by}\otimes\pi_{\by}$.
	Define a function space $\mathcal{S}$ of the transport map, then we seek transport map by minimizing the MMD loss. Then the problem is equivalent to the following optimization problem,
	\begin{equation}\label{eq:3.5}
		\mathcal{T}^* = \underset{\mathcal{T}\in \mathcal{S}}{\min}\ \underset{f\in \mathcal{F}}{\sup}\left(\mathbb{E}_{\mathbf{z}\sim\pi_1^{\mathcal{T}}}[f(\mathbf{z})]-\mathbb{E}_{\mathbf{z}^{\prime}\sim\pi_2}[f(\mathbf{z}^{\prime})]\right).
	\end{equation}
	Al-Jarrah et al., \cite{al-jarrahNonlinearFilteringBrenier2025}, proposed an optimization problem similar to \eqref{eq:3.5}. Here, we extend it to a MMD-based formulation for easy computation. The min-max problem can be reformulated with the reproducing property of RKHS, which gives
	\begin{equation}\label{eq:3.6}
		\operatorname{MMD}^2[\pi_1^{\mathcal{T}}, \pi_2, \mathcal{F}] = \mathbb{E}_{\tilde{\bz}, \tilde{\bz}^{\prime}\sim \pi_1^{\mathcal{T}}}[k(\tilde{\bz}, \tilde{\bz}^{\prime})] - 2\mathbb{E}_{\tilde{\bz}\sim\pi_1^{\mathcal{T}}, \bar{\bz}\sim \pi_2}[k(\tilde{\bz}, \bar{\bz})] + \mathbb{E}_{\bar{\bz}, \bar{\bz}^{\prime}\sim \pi_2}[k(\bar{\bz}, \bar{\bz}^{\prime})].
	\end{equation}
	
	\subsection{Empirical approximation of transport maps}\label{sec:3.1}
	
	Suppose the two copies $\mathbf{Z} = (\bX, \bY)$ and $\bar{\mathbf{Z}} = (\bar{\bX}, \bar{\bY})$ of $N$ coupled samples of the joint distribution $P(\bx, \by)$ are generated via nonlinear system \eqref{eq:2.1}-\eqref{eq:2.2}, and then an independent copy $\tilde{\mathbf{Z}} = ({\bX}, \bar{\bY})$ of random variables $\bx$ and $\by$ can be constructed,
	\begin{equation*}
		\begin{aligned}
			\mathbf{Z} = (\bX, \bY) &=\{\bx^i, \by^i\}_{i=1}^N := \{\mathbf{z}^i\}_{i=1}^N,\\
			\bar{\mathbf{Z}} = (\bar{\bX}, \bar{\bY}) &=\{\bar{\bx}^i, \bar{\by}^i\}_{i=1}^N := \{\bar{\mathbf{z}}^i\}_{i=1}^N,\\
			\tilde{\mathbf{Z}} = ({\bX}, \bar{\bY}) &= \{{\bx}^i, \bar{\by}^i\}_{i=1}^N := \{\tilde{\mathbf{z}}^i\}_{i=1}^N.
		\end{aligned}
	\end{equation*}
	In this work, we consider constructing the transport map~$\mathcal{T}$ using neural networks, which enables efficient training and inference. 
	As shown in \Cref{fig:transport-structure}, we first employ a prior encoder and an observation encoder to extract key features from the prior particles and observations, respectively. Then we use a coupling network to integrate the information from both sources.
	\begin{figure}[!htbp]
		\centering
		\includegraphics[width=.8\textwidth]{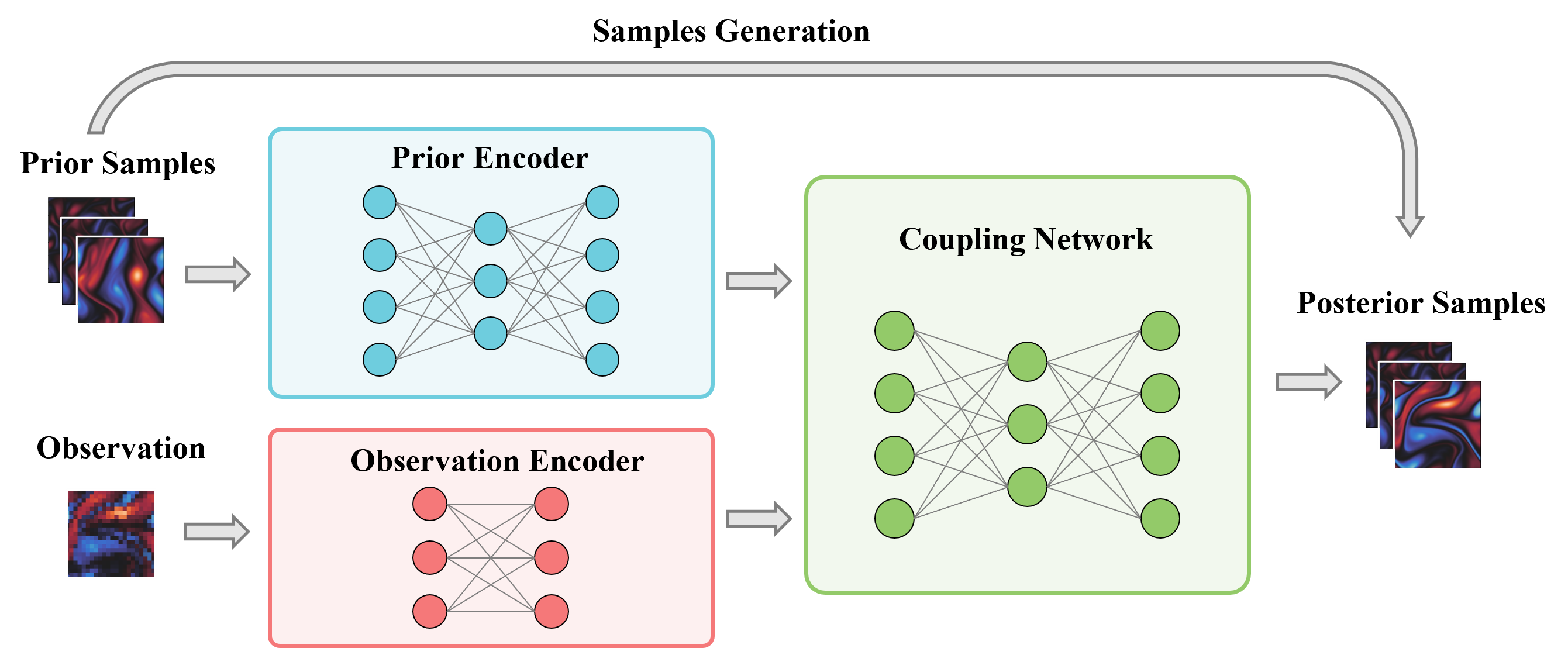}
		\caption{The network architecture for the component map $T$ of the transport map $\mathcal{T}$: prior particles and observations are first preprocessed through two independent encoders, and then fed into a coupling network for posterior particles.}
		\label{fig:transport-structure}
	\end{figure}
	For simplicity, we can construct a nudging-like transport map $\mathcal{T}^{\theta}$ with the form of
	\begin{equation}\label{eq:3.7}
		\mathcal{T}^{\theta}({\bX}, \bar{\bY}) = \left[
		\begin{matrix}
			{\bX} + \mathbf{T}^{\theta}(\bar{\bY}-\mathbf{H}({\bX}))\\
			\bar{\bY}
		\end{matrix}\right],
	\end{equation}
	where $\mathbf{T}^{\theta}$ is defined by a neural network with parameters $\theta$. Then the transport map $\mathbf{T}$ is obtained by minimizing the MMD in \eqref{eq:3.6},
	\begin{equation}\label{eq:3.8}
		\mathcal{T}^{\theta^*} = \underset{\mathcal{T}^{\theta}}{\arg\min} \operatorname{MMD}^2[\mathcal{T}^{\theta}_{\sharp} (\pi_{\bx}\otimes \pi_{\by}), \pi_{\bx, \by}, \mathcal{F}]. 
	\end{equation}
	To numerically solve the optimization problem \eqref{eq:3.8}, the objective function \eqref{eq:3.6} can be approximated with the two copies of samples $(\bX, \bY)$ and $(\bar{\bX}, \bar{\bY})$, 
	\begin{equation}\label{eq:3.9}
		\begin{aligned}
			Loss \approx& \widehat{\operatorname{MMD}}^2[\mathcal{T}^{\theta}_{\sharp} (\pi_{\bx}\otimes \pi_{\by}), \pi_{\by\mid \bx}\otimes \pi_{\bx}; \mathcal{F}]\\
			=& \frac{1}{N^2}\sum_{i=1}^{N}\sum_{j=1}^{N}k(\mathcal{T}^{\theta}(\tilde{\mathbf{z}}^i), \mathcal{T}^{\theta}(\tilde{\mathbf{z}}^j))-\frac{2}{N^2}\sum_{i=1}^{N}\sum_{j=1}^{N}k(\mathcal{T}^{\theta}(\tilde{\mathbf{z}}^i), {\mathbf{z}}^j) \\
			&+ \frac{1}{N^2}\sum_{i=1}^{N}\sum_{j=1}^{N}k({\mathbf{z}}^i, {\mathbf{z}}^j). 
		\end{aligned}
	\end{equation}
	By minimizing the approximate objective function \eqref{eq:3.9}, we obtain the transport map $\mathcal{T}^*$. 
	Then we can sample from the posterior distribution by fixing $\mathbf{y} = \mathbf{y}^o$,
	\begin{equation}
		\bx^{post} =  T^{\theta}(\bx, \by^{o}),
	\end{equation}
	where $\mathbf{y}^o$ is the value of observation. \Cref{alg:1} presents the pseudocode for coupling-informed transport filter method. 
	
	\begin{algorithm}
		\renewcommand{\algorithmicrequire}{\textbf{Input:}}
		\renewcommand{\algorithmicensure}{\textbf{Output:}}
		\caption{Coupling-informed transport filter(TFCP)}
		\begin{algorithmic}[1]
			\REQUIRE Given ensemble $\{\bx_{k-1}^{(i)}\}_{i=1}^N$ at previous time step $k-1$, current observation $\by^o_k$, discrete dynamical model $\mathbf{M}_{k}(\cdot)$, observation operator $\mathbf{H}_k(\cdot)$, model uncertainty $\pi(\boldsymbol{\eta}_{k})$;
			\FOR {$i=1,\cdots, N$} 
			\STATE $\hat{\bx}_{k}^{(i)} = \mathbf{M}_{k}(\bx_{k-1}^{(i)}) + \boldsymbol{\eta}_{k}^{(i)}$; 
			\STATE $\hat{\by}_{k}^{(i)} = \mathbf{H}_{k}(\hat{\bx}_{k}^{(i)}) + \boldsymbol{\varepsilon}_k$
			\ENDFOR 
			\STATE Randomly shuffle the forecast ensemble to generate $(\bX, \bY)$ and $({\bX}, \bar{\bY})$;
			\STATE Initialize the parameter $\theta$ of transport map $\mathcal{T}^{\theta}$ defined as \eqref{eq:3.7};
			\REPEAT
			\FOR{$i=1,\cdots, N$}
			\STATE $\hat{\bx}_{k}^{(i)} = {T}^{\theta}(\hat{\bx}_{k}^{(i)}, \bar{\by}_{k}^{(i)})$;
			\ENDFOR
			\STATE Calculate $\operatorname{MMD}^2$ loss by \eqref{eq:3.9};
			\STATE Update the parameter $\theta$ by minimizing the $\operatorname{MMD}^2$ loss.
			\UNTIL{Stop when criterion met}
			\FOR{$i=1,\cdots, N$}
			\STATE ${\bx}_{k}^{(i)} = \hat{\bx}_{k}^{(i)} + {T}^{*}({\by}_{k}^{o}-\mathbf{H}_k(\hat{\bx}_{k}^{(i)}))$;
			\ENDFOR
			\ENSURE Approximated posterior ensemble $\{{\bx}_{k}^{(i)}\}_{i=1}^N$;
		\end{algorithmic}
		\label{alg:1}
	\end{algorithm}
	
	{
		Assume that $\mathcal{H}$ is the reproducing kernel Hilbert space associated with a Gaussian kernel $k(\bz, \bz^{\prime})=\exp(-\|\bz-\bz^{\prime}\|^2/b_w^2)$, where $\bz=(\bx, \by)$ and $b_w$ is the bandwidth. Assume that the loss function \eqref{eq:3.9} is sufficiently minimized to a global optimum. Then a convergence of the approximate posterior distribution to the true posterior distribution is presented as follow. 
		\begin{thm}
			Assume the transport map defined by $\mathcal{T}^{\theta}(\bx, \bar{\by})=\left[\mathbf{T}^{\theta}(\bx, \bar{\by}), \bar{\by}\right]^{\top}$ in \eqref{eq:3.2} and denote $P_{\mathcal{T}^{\theta}}=\mathcal{T}^{\theta}_{\sharp}(\pi_{\bx}\otimes \pi_{\by})$ and $Q=\pi_{\bx, \by}$. When the kernel of MMD is Gaussian, for any $\delta>0$, with probability at least $1-\delta$, the following inequality holds 
			\begin{equation}\label{eq:3.11}
				\begin{aligned}
					\mathbb{E}_{p(\by)}{\operatorname{MMD}}\Big[\mathbf{T}^{\theta}(\cdot, \by)_{\sharp}P(\bx),P(\bx\mid \by);\mathcal{F}\Big] \le& \frac{\sqrt{2}}{b_w}\mathbb{E}_{p(\bx)p(\by)}\Big[\|\mathbf{T}^{\theta^{*}}(\bx, \by)-\mathbf{T}^{*}(\bx, \by)\|\Big]\\
					&+ 2\left(1+\sqrt{\ln\frac{1}{\delta}}\right)\sqrt{\frac{1}{N}},
				\end{aligned}
			\end{equation}
			where $\mathcal{T}^*$ is the truth transport that satisfy $\mathcal{T}^{{*}}_{\sharp}(\pi_{\bx}\otimes \pi_{\by})=\pi_{\bx, \by}$. 
		\end{thm}
		\begin{proof}
			The RKHS $\mathcal{H}$ associated with a Gaussian kernel admits a decomposition as a tensor product of two RKHS $\mathcal{H}=\mathcal{H}_{\bx}\otimes \mathcal{H}_{\by}$, where both $\mathcal{H}_{\bx}$ and $\mathcal{H}_{\by}$ have Gaussian kernel with bandwidth $b_w$. For any $f_1\in \mathcal{F}_{\bx}$ and $f_2 = \mathcal{F}_{\by}$, the product $f(\bx, \by)=f_1(\bx)f_2(\by)\in \mathcal{F}$. Then
			\begin{equation*}
				\begin{aligned}
					{\operatorname{MMD}}[P_{\mathcal{T}^{\theta}}, Q, \mathcal{F}] &= \sup_{f \in \mathcal{F}} \Big\{\mathbb{E}_{p(\bx)p(\by)}\left[f(\mathbf{T}^{\theta}(\bx, \by),\by)\right]-\mathbb{E}_{p(\bx,\by)}\left[f(\bx, \by)\right]\Big\}\\
					&\ge \sup_{\substack{f_1\in \mathcal{F}_{\bx},\\ f_2\in \mathcal{F}_{\by}}}\Big\{\mathbb{E}_{p(\bx)p(\by)} f_1(\mathbf{T}^{\theta}(\bx, \by))f_2(\by)-\mathbb{E}_{p(\bx, \by)} f_1(\bx)f_2(\by)\Big\}\\
					&=\sup_{\substack{f_1\in \mathcal{F}_{\bx},\\ f_2\in \mathcal{F}_{\by}}}\Big\{\mathbb{E}_{p(\by)}\left[\left(\mathbb{E}_{p(\bx)}\left[f_1(\mathbf{T}^{\theta}(\bx, \by))\right]-\mathbb{E}_{p(\bx\mid \by)}\left[f_1(\bx)\right]\right)f_2(\by)\right]\Big\}.
				\end{aligned}
			\end{equation*} 
			According to the reproducing property of RKHS and the kernel of $\mathcal{H}_{\by}$ is Gaussian, we have $\sup_{f_2\in \mathcal{F}_{\by}}f_2(\by)=\|k(\by, \cdot)\|_{\mathcal{H}_y}=1$. Thus
			\begin{equation*}
				\begin{aligned}
					{\operatorname{MMD}}&[P_{\mathcal{T}^{\theta}}, Q, \mathcal{F}] \\
					&\ge \sup_{f_1\in \mathcal{F}_{\bx}}\left\{\mathbb{E}_{p(\by)}\left[\Big(\mathbb{E}_{p(\bx)}\left[f_1(\mathbf{T}^{\theta}(\bx, \by))\right]-\mathbb{E}_{p(\bx\mid \by)}\left[f_1(\bx)\right]\Big)\sup_{f_2\in \mathcal{F}_{\by}}f_2(\by)\right]\right\}\\
					&\ge \mathbb{E}_{p(\by)}\left[\sup_{f_1\in \mathcal{F}_{\bx}}\Big\{\mathbb{E}_{p(\bx)}\left[f_1(\mathbf{T}^{\theta}(\bx, \by))\right]-\mathbb{E}_{p(\bx\mid \by)}\left[f_1(\bx)\right]\Big\}\right]\\
					&=\mathbb{E}_{p(\by)}\Big[\operatorname{MMD}\left[\mathbf{T}^{\theta}(\cdot, \by)_{\sharp}P(\bx),P(\bx\mid \by);\mathcal{F}_x\right]\Big].
				\end{aligned}
			\end{equation*}
			Applying the triangle inequality yields
			\begin{equation}\label{eq:3.12}
				\begin{aligned}
					\mathbb{E}_{p(\by)}\big[\operatorname{MMD}[\mathbf{T}^{\theta}(\cdot, \by)_{\sharp}P(\bx),&P(\bx\mid \by);\mathcal{F}_x]\big] \le {\operatorname{MMD}}[P_{\mathcal{T}^{\theta}}, Q, \mathcal{F}]\\
					&\le {\operatorname{MMD}}[P_{\mathcal{T}^{\theta^{*}}}, Q, \mathcal{F}] + {\operatorname{MMD}}[P_{\mathcal{T}^{\theta}}, P_{\mathcal{T}^{\theta^{*}}}, \mathcal{F}].
				\end{aligned}
			\end{equation}
			Since the kernel of MMD is chosen to be Gaussian, then $0 < k(\bz, \bz^{\prime}) \le 1$. According to the Theorem 14 of \cite{grettonKernelMethodTwoSampleProblem2006}, the convergence rate of empirical approximation \eqref{eq:3.7} is
			\begin{equation}\label{eq:3.13}
				\operatorname{Pr}\left(\left|{\operatorname{MMD}}[P_{\mathcal{T}^{\theta}}, P_{\mathcal{T}^{\theta^{*}}}, \mathcal{F}]-\widehat{\operatorname{MMD}}[P_{\mathcal{T}^{\theta}}, P_{\mathcal{T}^{\theta^{*}}}, \mathcal{F}]\right|>2\sqrt{\frac{1}{N}}+\epsilon\right)\le\exp\left(-\frac{\epsilon^2 N}{4}\right).
			\end{equation} 
			Equation \eqref{eq:3.13} means that,
			\begin{equation}\label{eq:3.14}
				\left|{\operatorname{MMD}}[P_{\mathcal{T}^{\theta}}, P_{\mathcal{T}^{\theta^{*}}}, \mathcal{F}]-\widehat{\operatorname{MMD}}[P_{\mathcal{T}^{\theta}}, P_{\mathcal{T}^{\theta^{*}}}, \mathcal{F}]\right| \le  2\left(1+\sqrt{\ln\frac{1}{\delta}}\right)\sqrt{\frac{1}{N}}
			\end{equation}
			hold for any $\delta>0$, with probability at least $1-\delta$. Since we assume that the empirical loss function MMD has been fully optimized, i.e., $\widehat{\operatorname{MMD}}[P_{\mathcal{T}^{\theta}}, P_{\mathcal{T}^{\theta^{*}}}, \mathcal{F}]=0$ , it follows that
			\begin{equation}\label{eq:3.15}
				{\operatorname{MMD}}[P_{\mathcal{T}^{\theta}}, P_{\mathcal{T}^{\theta^{*}}}, \mathcal{F}] \le 2\left(1+\sqrt{\ln\frac{1}{\delta}}\right)\sqrt{\frac{1}{N}}.
			\end{equation}
			By the Lipschitz continuity of the Gaussian kernel function, we have
			\begin{equation}\label{eq:3.16}
				\begin{aligned}
					{\operatorname{MMD}}[P_{\mathcal{T}^{\theta^{*}}}, Q, \mathcal{F}] &= \left\|\mathbb{E}_{p(\bx)p(\by)}\left[k(\mathcal{T}^{\theta^{*}}(\bz), \cdot)\right] - \mathbb{E}_{p(\bx)p(\by)}\left[k(\mathcal{T}^{*}(\bz), \cdot)\right]\right\|_{\mathcal{H}}\\
					&\le \mathbb{E}_{p(\bx)p(\by)}\Big[\left\|k(\mathcal{T}^{\theta^{*}}(\bz), \cdot)-k(\mathcal{T}^{*}(\bz), \cdot)\right\|_{\mathcal{H}}\Big]\\
					&\le \mathbb{E}_{p(\bx)p(\by)}\Big[2-2\exp\left(-\|\mathcal{T}^{\theta^{*}}(\bz)-\mathcal{T}^{*}(\bz)\|^2/b_w^2\right)\Big]^{\frac{1}{2}}\\
					&\le \frac{\sqrt{2}}{b_w}\mathbb{E}_{p(\bx)p(\by)}\Big[\|\mathcal{T}^{\theta^{*}}(\bz)-\mathcal{T}^{*}(\bz)\|\Big]\\
					&=\frac{\sqrt{2}}{b_w}\mathbb{E}_{p(\bx)p(\by)}\Big[\|\mathbf{T}^{\theta^{*}}(\bx, \by)-\mathbf{T}^{*}(\bx, \by)\|\Big].
				\end{aligned}
			\end{equation}
			Substituting \eqref{eq:3.15} and \eqref{eq:3.16} into \eqref{eq:3.12} yields
			\begin{equation*}
				\begin{aligned}
					\mathbb{E}_{p(\by)}{\operatorname{MMD}}[\mathbf{T}^{\theta}(\cdot, \by)_{\sharp}P(\bx),P(\bx\mid \by);\mathcal{F}] \le& \frac{\sqrt{2}}{b_w}\mathbb{E}_{p(\bx)p(\by)}\left[\|\mathbf{T}^{\theta^{*}}(\bx, \by)-\mathbf{T}^{*}(\bx, \by)\|\right]\\
					&+ 2\left(1+\sqrt{\ln\frac{1}{\delta}}\right)\sqrt{\frac{1}{N}}.
				\end{aligned}
			\end{equation*}
			This completes the proof.
		\end{proof}
		
		The first term in Inequality \eqref{eq:3.11} represents the approximation error of the transport map $\mathcal{T}^{\theta}$. This term becomes negligible when the parameterized function space is sufficiently rich to contain the true transport map. The second term characterizes the mean convergence rate between the approximate posterior distribution and the truth posterior distribution, which is of  $\mathcal{O}(1/\sqrt{N})$  with respect to the sample size  $N$ .
		
	}
	
	\subsection{Transport with gradient flow}
	
	In \Cref{sec:3.1}, we parameterize the transport map using a neural network and then minimize the MMD loss between the approximate coupling and the reference coupling via gradient descent. However, the choice of parameterization for the transport map can directly affect the accuracy of the approximation, and the non-convexity of the MMD as a loss function may cause the optimization to become trapped in a local minima. Inspired by stein variational gradient descent \cite{liuSteinVariationalGradient2016} and its applications in filtering \cite{pulidoSequentialMonteCarlo2019}, we consider using gradient flows to construct the transport map,
	\begin{equation}\label{eq:3.17}
		\begin{cases}
			\frac{\mathrm{d} {\bx}_{\tau}}{\mathrm{d}\tau} = \mathbf{v}_{\tau}({\bx}_{\tau}, \bar{\by}_{\tau}), \tau>0,\\
			\bar{\by}_{\tau} = \bar{\by},
		\end{cases}
	\end{equation} 
	where $\tau$ is a pseudo time and $\mathbf{v}_{\tau}$ is the drift term or velocity field that push particles toward the posterior. Arbel et al., \cite{arbelMaximumMeanDiscrepancy2019}, constructed the gradient flow of the standard MMD and studied its convergence properties. Consider a forward Euler scheme of gradient flow \eqref{eq:3.17},
	\begin{equation}\label{eq:3.18}
		\mathcal{T}_{\tau+\epsilon}({\bx}_{\tau}, \bar{\by}_{\tau}) = \left[
		\begin{matrix}
			{\bx}_{\tau} + \epsilon\mathbf{v}({\bx}_{\tau}, \bar{\by}_{\tau}) \\
			\bar{\by}_{\tau}
		\end{matrix}\right],
	\end{equation}
	where $\epsilon>0$ is the stepsize. Start from the indepedent coupling $({\bx}_{0}, \bar{\by}_{0})\sim \pi_{\bx}\otimes \pi_{\by}$, a sequence of local transformation push forward measure to the reference coupling $({\bx}_{\infty}, \bar{\by}_{\infty})\sim \pi_{\bx, \by}$. Then, the goal is to find the optimal velocity field  $\mathbf{v}_{\tau}$  that induces the steepest descent of the MMD.
	
	Suppose $\tilde{\bz},\tilde{\bz}^{\prime} \sim \pi_{\bx}\otimes \pi_{\by}$ and $\bar{\bz}, \bar{\bz}^{\prime}\sim \pi_{\bx, \by}$, the MMD between the reference coupling and the independent coupling after transformation $\mathcal{T}$ defined in \eqref{eq:3.18} can be expressed by
	\begin{equation}\label{eq:3.19}
		\operatorname{MMD}^2(\mathcal{T}) = \mathbb{E}_{\tilde{\bz}, \tilde{\bz}^{\prime}}[k(\mathcal{T}(\tilde{\bz}), \mathcal{T}(\tilde{\bz}^{\prime}))] - 2\mathbb{E}_{\tilde{\bz}, \bar{\bz}}[k(\mathcal{T}(\tilde{\bz}), \bar{\bz})] + \mathbb{E}_{\bar{\bz},\bar{\bz}^{\prime}}[k(\bar{\bz},\bar{\bz}^{\prime})].
	\end{equation}
	According to the definition of Gateaux derivative of a functional $F$,
	\begin{equation*}
		D_hF(\bx) = \lim_{\epsilon\to 0} \frac{F(\bx+\epsilon h(\bx))-F(\bx)}{\epsilon}.
	\end{equation*}
	The Gateaux derivative of MMD \eqref{eq:3.19} is 
	\begin{equation}\label{eq:3.20}
		\begin{aligned}
			D_{\mathbf{v}}\operatorname{MMD}^2(\mathcal{T}) =& -\frac{2}{b_w^2}\mathbb{E}_{\tilde{\bz}, \tilde{\bz}^{\prime}}\left[ k(\tilde{\bz}, \tilde{\bz}^{\prime})(\tilde{\bx}-\tilde{\bx}^{\prime})^{\top}(\mathbf{v}(\tilde{\bz})-\mathbf{v}(\tilde{\bz}^{\prime})) \right]\\
			&+\frac{4}{b_w^2}\mathbb{E}_{\tilde{\bz}, \bar{\bz}}\left[ k(\tilde{\bz}, \bar{\bz})(\tilde{\bx}-\bar{\bx})^{\top}\mathbf{v}(\tilde{\bz}) \right],
		\end{aligned}
	\end{equation}
	where $b_w$ is the bandwidth of kernel $k(\cdot, \cdot)$. Equation \eqref{eq:3.20} gives the the directional derivative of the MMD along the direction $\mathbf{v}$. To obtain the steepest descent direction of the MMD, we constrain each component of $\mathbf{v}(\cdot)$ to be choosen from the unit ball of a RKHS $\mathcal{F}_{\gamma}$ associated with a kernel function $k_{\gamma}(\bz_1, \bz_2)=\exp\left(-{\|\bz_1-\bz_2\|^2}/{\gamma^2}\right)$. Then, with the reproducing property of RKHS, it gives
	\begin{equation}\label{eq:3.21}
		\mathbf{v}(\bz) = \left[
		\begin{matrix}
			\langle k_{\gamma}(\bz, \cdot), \mathbf{v}_1(\cdot)\rangle_{\mathcal{F}_{\gamma}}\\
			\vdots\\
			\langle k_{\gamma}(\bz, \cdot), \mathbf{v}_n(\cdot)\rangle_{\mathcal{F}_{\gamma}}
		\end{matrix}
		\right]:=\langle k_{\gamma}(\bz, \cdot), \mathbf{v}(\cdot)\rangle_{\mathcal{F}_{\gamma}}.
	\end{equation}
	Substituting \eqref{eq:3.21} into \eqref{eq:3.20} yields
	\begin{equation}\label{eq:3.22}
		\begin{aligned}
			D_{\mathbf{v}}\operatorname{MMD}^2(\mathcal{T}) =&\left\langle -\frac{2}{b_w^2}\mathbb{E}_{\tilde{\bz}, \tilde{\bz}^{\prime}}\left[ k(\tilde{\bz}, \tilde{\bz}^{\prime})(k_{\gamma}(\tilde{\bz}, \cdot)-k_{\gamma}(\tilde{\bz}^{\prime}, \cdot))(\tilde{\bx}-\tilde{\bx}^{\prime}) \right]\right.\\
			&\left.+\frac{4}{b_w^2}\mathbb{E}_{\tilde{\bz}, \bar{\bz}}\left[ k(\tilde{\bz}, \bar{\bz})k_{\gamma}(\tilde{\bz}, \cdot)(\tilde{\bx}-\bar{\bx}) \right], \mathbf{v}(\cdot) \right\rangle_{\mathcal{F}_{\gamma}}.
		\end{aligned}
	\end{equation}
	Equation \eqref{eq:3.22} holds true for all $\mathbf{v}\in \mathcal{F}_{\gamma}$ such that $\|\mathbf{v}\|_{\mathcal{F}_{\gamma}}\le 1$. Then $D_{\mathbf{v}}\operatorname{MMD}^2 = \langle \nabla_{\bx}\operatorname{MMD}^2 , \mathbf{v} \rangle_{\mathcal{F}_{\gamma}}$ by the definition of gradient. Thus, we have
	\begin{equation}\label{eq:3.23}
		\begin{aligned}
			\nabla_{\bx}\operatorname{MMD}^2(\bz) = & -\frac{2}{b_w^2}\mathbb{E}_{\tilde{\bz}, \tilde{\bz}^{\prime}}\Big[ k(\tilde{\bz}, \tilde{\bz}^{\prime})(k_{\gamma}(\tilde{\bz}, \bz)-k_{\gamma}(\tilde{\bz}^{\prime}, \bz))(\tilde{\bx}-\tilde{\bx}^{\prime}) \Big]\\
			&+\frac{4}{b_w^2}\mathbb{E}_{\tilde{\bz}, \bar{\bz}}\Big[ k(\tilde{\bz}, \bar{\bz})k_{\gamma}(\tilde{\bz}, \bz)(\tilde{\bx}-\bar{\bx}) \Big].
		\end{aligned}
	\end{equation}
	Here, Equation \eqref{eq:3.23} represents the steepest descent direction of the MMD, which can be used to construct the local gradient flow update in \eqref{eq:3.18},
	\begin{equation}
		{\bx}_{\tau+\epsilon} = {\bx}_{\tau} - \epsilon\nabla_{\bx}\operatorname{MMD}^2(\bz_{\tau}),
	\end{equation}
	where $\bz_{\tau}=(\bx_{\tau}, \by_{\tau})$. Note that the gradient computation of the MMD \eqref{eq:3.23} involves two distinct kernel functions: the kernel  $k(\cdot, \cdot)$  originates from the MMD itself and is used to capture correlations within the dataset, while the kernel  $k_{\gamma}(\cdot, \cdot)$  arises from Equation \eqref{eq:3.21} and characterizes the correlation between the input and the data points in the dataset.
	
	Based on the setting of empirical approximation in \Cref{sec:3.1}, we can compute the expectation in \eqref{eq:3.23} with Monte-Carlo integration, 
	\begin{equation}
		\begin{aligned}
			\nabla_{\bx}\operatorname{MMD}^2(\bz) \approx & -\frac{2}{b_w^2}\frac{1}{N^2}\sum_{i=1}^N\sum_{j=1}^N\Big[ k(\tilde{\bz}^i, \tilde{\bz}^j)(k_{\gamma}(\tilde{\bz}^i, \bz)-k_{\gamma}(\tilde{\bz}^j, \bz))(\tilde{\bx}^i-\tilde{\bx}^j) \Big]\\
			&+\frac{4}{b_w^2}\frac{1}{N^2}\sum_{i=1}^N\sum_{j=1}^N\Big[ k(\tilde{\bz}^i, \bar{\bz}^j)k_{\gamma}(\tilde{\bz}^i, \bz)(\tilde{\bx}^i-\bar{\bx}^j) \Big].
		\end{aligned}
	\end{equation}
	By combining gradient flows with coupling-based transport filtering, we construct a likelihood-free and training-free filtering method. \Cref{alg:2} provides its pseudocode.
	
	\begin{algorithm}
		\renewcommand{\algorithmicrequire}{\textbf{Input:}}
		\renewcommand{\algorithmicensure}{\textbf{Output:}}
		\caption{Coupling-informed transport filter with gradient flow (TFCP-GF)}
		\begin{algorithmic}[1]
			\REQUIRE Given ensemble $\{\bx_{k-1}^{(i)}\}_{i=1}^N$ at previous time step $k-1$, current observation $\by^o_k$, discrete dynamical model $\mathbf{M}_{k}(\cdot)$, observation operator $\mathbf{H}_k(\cdot)$, model uncertainty $\pi(\boldsymbol{\eta}_{k})$;
			\FOR {$i=1,\cdots, N$} 
			\STATE $\hat{\bx}_{k}^{(i)} = \mathbf{M}_{k}(\bx_{k-1}^{(i)}) + \boldsymbol{\eta}_{k}^{(i)}$; 
			\STATE $\hat{\by}_{k}^{(i)} = \mathbf{H}_{k}(\hat{\bx}_{k}^{(i)}) + \boldsymbol{\varepsilon}_k$
			\ENDFOR 
			\STATE Randomly shuffle the forecast ensemble to generate $(\bX, \bY)$ and $(\bar{\bX}, \bY)$;
			\STATE Initialize posterior ensemble $\{{\bx}_{k}^{(i)}\}_{i=1}^N$ as $\{\hat{\bx}_{k}^{(i)}\}_{i=1}^N$
			\REPEAT
			\STATE Create update function $\mathbf{v}_{\tau}(\bx, \by) = \nabla_{\bx}\operatorname{MMD}^2(\bx, \by)$ with $\{\hat{\bx}_{k,\tau}^{(i)}, \hat{\by}_{k,\tau}^{(i)}\}_{i=1}^N$;
			\FOR{$i=1,\cdots, N$}
			\STATE $\hat{\bx}_{k,\tau+\epsilon}^{(i)} = \hat{\bx}_{k,\tau}^{(i)} - \epsilon \mathbf{v}_{\tau}(\hat{\bx}_{k,\tau}^{(i)}, \hat{\by}_{k}^{(i)})$; 
			\STATE ${\bx}_{k,\tau+\epsilon}^{(i)} = {\bx}_{k,\tau}^{(i)} - \epsilon \mathbf{v}_{\tau}({\bx}_{k,\tau}^{(i)}, \by^o_k)$; 
			\ENDFOR
			\UNTIL{Stopping criterion met}
			\ENSURE Approximated posterior ensemble $\{{\bx}_{k}^{(i)}\}_{i=1}^N$;
		\end{algorithmic}
		\label{alg:2}
	\end{algorithm}
	
	\subsection{Localization of transport maps}
	
	The treatment of high-dimensional systems remains significantly challenging in data assimilation. This is primarily due to the so-called curse of dimensionality, where the number of particles needed for accurate distribution estimation increases exponentially with the state space dimension. To mitigate these challenges and achieve robust posterior approximations with limited ensemble sizes, traditional filtering methods—such as the EnKF—incorporate localization and covariance inflation to enhance numerical stability and estimation accuracy.
	
	In this section, we introduce the idea of domain localization to the proposed coupling-informed transport filtering by imposing a local structure on the transport map $\mathcal{T}$ defined in Equation \eqref{eq:3.2}. We decompose the full high-dimensional state $\bx\in\mathbb{R}^{n}$ into $K$ low-dimensional local blocks,
	\begin{equation*}
		\bx = \{\bx^i_{loc}\}_{i=1}^K,\ \bx^i_{loc}\in\mathbb{R}^{n_i},\ \sum_{i=1}^{K}n_i=n. 
	\end{equation*}
	In the simplest configuration, each local block consists of a single state component, corresponding to $n_i=1, i=1\cdots n$. For the update of the states within each local block, we utilize information from states and observations $({\bx}_{eloc}^i, {\by}_{eloc}^i)$ in an extended local block formed by oversampling a region of radius $r_{\text{loc}}$ around the block. \Cref{fig:localization} illustrates the concept of domain localization of sparse observation in high dimensional problem, the data assimilation problem in high-dimensional state space is transformed into multiple low-dimensional data assimilation problems within extended local blocks.
	
	\begin{figure}[!htbp]
		\centering
		\subfigure{
			\begin{minipage}[t]{0.45\textwidth}
				\centering
				\includegraphics[width=.85\textwidth]{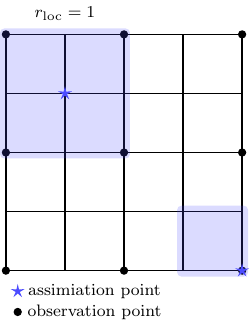}
			\end{minipage}
		}
		\subfigure{
			\begin{minipage}[t]{0.45\textwidth}
				\centering
				\includegraphics[width=.85\textwidth]{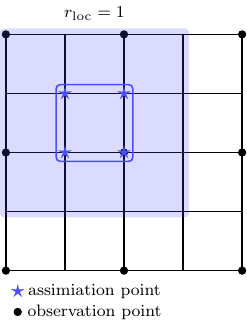}
			\end{minipage}
		}
		\caption{Domain localization of sparse observation in high dimensional problem}
		\label{fig:localization}
	\end{figure}
	
	For $i$-th extended local block, suppose $({\bX}_{eloc}^i, {\bY}_{eloc}^i)$ and $(\bar{\bX}_{eloc}^i, \bar{\bY}_{eloc}^i)$ are two independent coupled samples of $P(\bx_{eloc}^i, \by_{eloc}^i)$. We construct a local transport map $\mathcal{T}^i_{eloc}$, which pushes $P({\bx}_{eloc}^i)\otimes P({\by}_{eloc}^i)$ to $P({\bX}_{eloc}^i, {\bY}_{eloc}^i)$. Like \eqref{eq:3.2}, the local transport map is defined as
	\begin{equation}
		\mathcal{T}^i_{eloc}({\bX}^i_{eloc}, \bar{\bY}^i_{eloc}) = \left[
		\begin{matrix}
			\mathbf{T}^i_{eloc}({\bX}^i_{eloc}, \bar{\bY}^i_{eloc})\\
			\bar{\bY}^i_{eloc}
		\end{matrix}\right].
	\end{equation}
	
	As stated in \Cref{thm:3.1}, the upper component of the mapping $\mathcal{T}^i_{eloc}$ yields an approximation to the posterior distribution $P(\bx_{eloc}^i\mid \by_{eloc}^i)$. Then, the coupling $(\bX^i_{loc}, \bY^i_{loc})$ of the local block is extracted from the couple $(\bX^i_{eloc}, \bY^i_{eloc})$ within its corresponding extended local domain. Denote this extraction or mapping process as $\mathcal{G}$.
	Then the global transport map can be expressed by 
	\begin{equation*}
		\mathcal{T}({\bX}, \bar{\bY}) = \operatorname{Rearrange}\left(\left[\begin{matrix}
			\mathcal{G}\circ\mathcal{T}^1_{eloc}({\bX}^1_{eloc}, \bar{\bY}^1_{eloc}),\\ 
			\mathcal{G}\circ\mathcal{T}^2_{eloc}({\bX}^2_{eloc}, \bar{\bY}^2_{eloc}),\\
			\vdots\\
			\mathcal{G}\circ\mathcal{T}^K_{eloc}({\bX}^K_{eloc}, \bar{\bY}^K_{eloc})\\
		\end{matrix}\right]\right),
	\end{equation*}
	where operator $\operatorname{Rearrange}(\cdot)$ reconstructs the original state vector by placing the updated states of all collected local blocks back to their respective positions in the original vector. Based on the aforementioned localization technique, different local transport maps can be trained in parallel, which enhances the computational efficiency of the proposed method for high-dimensional problems.
	
	\section{Numerical Examples}\label{sec:4}
	
	In this section, we present several numerical examples to demonstrate the performance of the proposed coupling-informed transport filter (TFCP) and the coupling-informed transport filter with gradient flow (TFCP-GF). In \Cref{sec:4.1} and \Cref{sec:4.2}, we evaluate the performance of TFCP and TFCP-GF in approximating the posterior distribution for static problems and dynamical systems, respectively. \Cref{sec:4.3} and \Cref{sec:4.4} focus on the accuracy of state estimation for stochastic dynamical systems and the effectiveness of domain localization in high-dimensional problems. The numerical examples are primarily benchmarked against traditional filtering methods, including SIRPF, EnKF, and LETKF. 
	
	To evaluate the accurancy of the inversion to the true state, the Average Root Mean Square Error (RMSE) is used,
	\begin{equation*}
		\text{RMSE} = \mathbb{E}\left[\frac{1}{T}\sum\limits_{k=1}^{T}\|\bar{\bx}_k - \bx_k^*\|_2 / \sqrt{n}\right],
	\end{equation*}
	where $\bar{\bx}_k \in \mathbb{R}^n$ is the ensemble mean and $\bx_k^*\in \mathbb{R}^n$ is the true state. To assess the quality of uncertainty estimation, we adopt the Coverage Probability (CP) as a metric,
	\begin{equation*}
		\text{CP} = \mathbb{E}\left[\frac{1}{n}\sum_{i=1}^n\left(\frac{1}{T}\sum_{k=1}^T \mathbb{I}_{s_i\le t_{\alpha}}\right)\right],\quad s_i = \left|\bar{\bx}_{k, i}-\bx^*_{k,i}\right|/\sqrt{\mathbf{C}_{k, (i,i)}^a},
	\end{equation*}
	where $\mathbf{C}_{k}$ is the ensemble covariance, $\mathbb{I}$ is indicator function and $t_{\alpha}$ is $\alpha$ quantiles of standard normal distribution. Ideally, when a filtering method provides well-calibrated uncertainty quantification, the coverage probability should be close to the nominal value $\alpha$. 
	
	\subsection{Static inverse problem}\label{sec:4.1}
	In the first numerical example, we implement two simple static inverse problems to present the ablility of coupling-informed transport filter in the approximation of the posterior. 
	
	Suppose $X\in \mathbb{R}$ is a random variable with Gaussian prior $\mathcal{N}(\mu_0, \sigma_0^2)$, and $H(\cdot)$ is a nonlinear observation operator defined by
	\begin{equation*}
		H(X) = X(X-1). 
	\end{equation*}
	Given an observation $Y^o$ corrupted by Gaussian noise $\varepsilon\sim \mathcal{N}(0, \sigma_1^2)$ , the goal of the inverse problem is to estimate the Bayesian posterior. Due to the quadratic nature of the observation operator, the Bayesian posterior is non-Gaussian, and in the case of certain observations, the posterior distribution becomes bimodal. For reference, the exact posterior can be calculated using the numerical integral method based on Bayes' rule,
	\begin{equation*}
		P(X\mid Y) \propto P(X)P(Y\mid X).
	\end{equation*}
	For simulation, we set parameters $\mu_0=0.5, \sigma_0=1.0, \sigma_1=0.5$ and the observation $Y^o=1.2$.  \Cref{fig:1d-static} illustrates the performance of TFCP in estimating the posterior distribution, showing a scatter plot of samples from the estimated joint distribution via the transport map, as well as the trajectories of particles transported from the prior to the approximate posterior.
	
	\begin{figure}[!htbp]
		\centering
		\subfigure{
			\begin{minipage}[t]{0.3\textwidth}
				\centering
				\includegraphics[width=\textwidth]{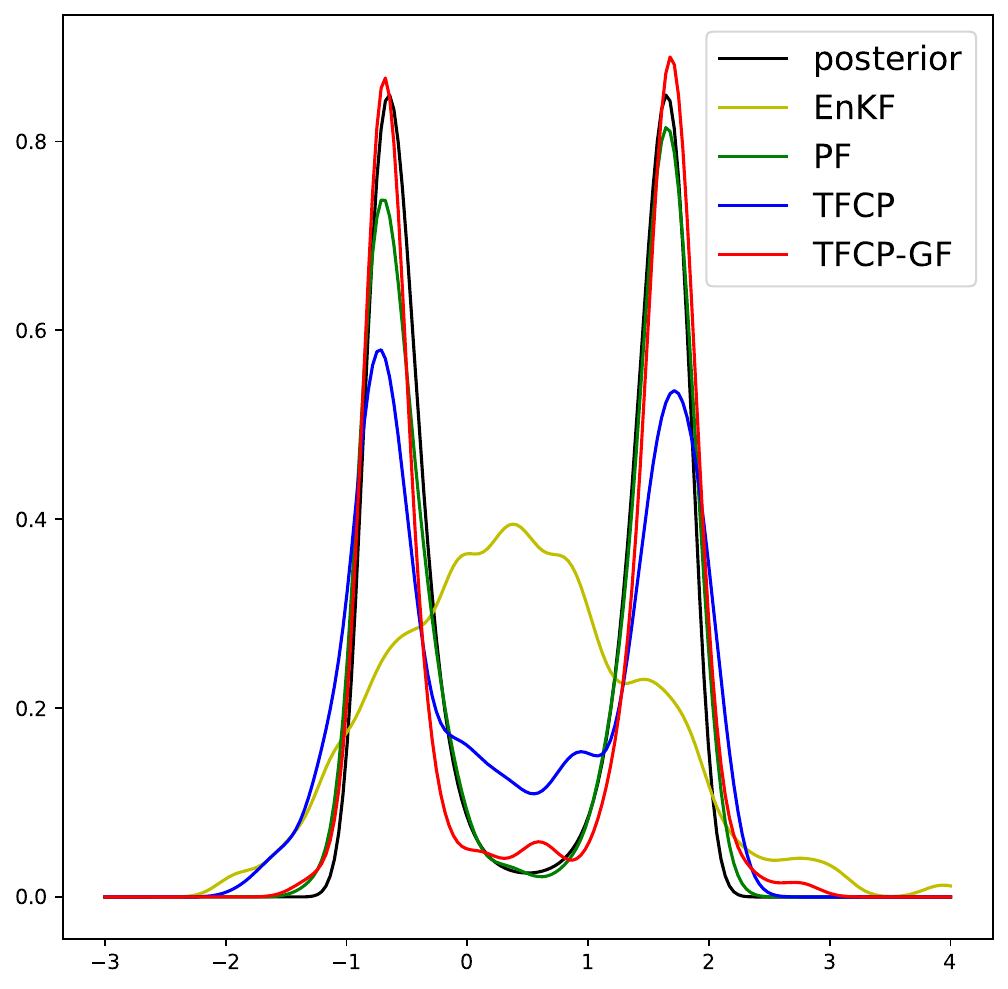}
			\end{minipage}
		}
		\subfigure{
			\begin{minipage}[t]{0.3\textwidth}
				\centering
				\includegraphics[width=\textwidth]{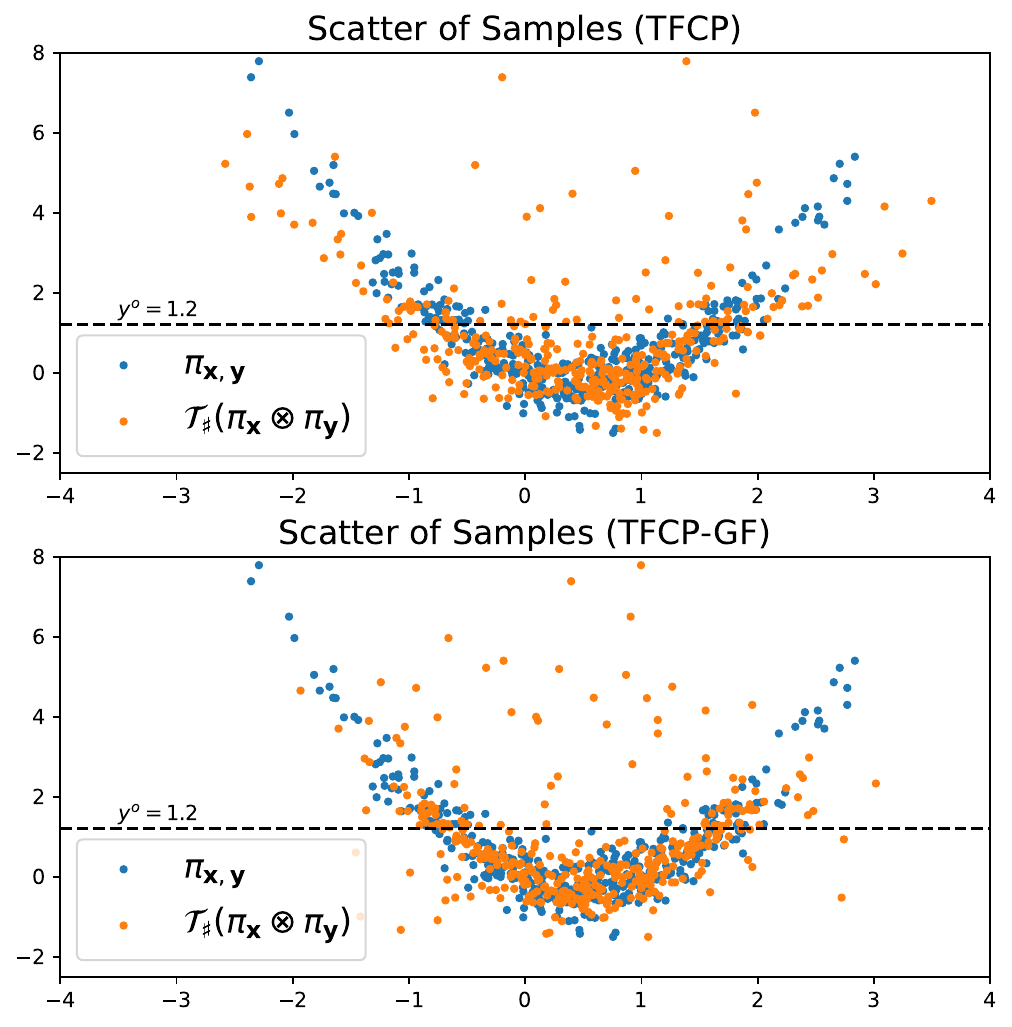}
			\end{minipage}
		}
		\subfigure{
			\begin{minipage}[t]{0.3\textwidth}
				\centering
				\includegraphics[width=\textwidth]{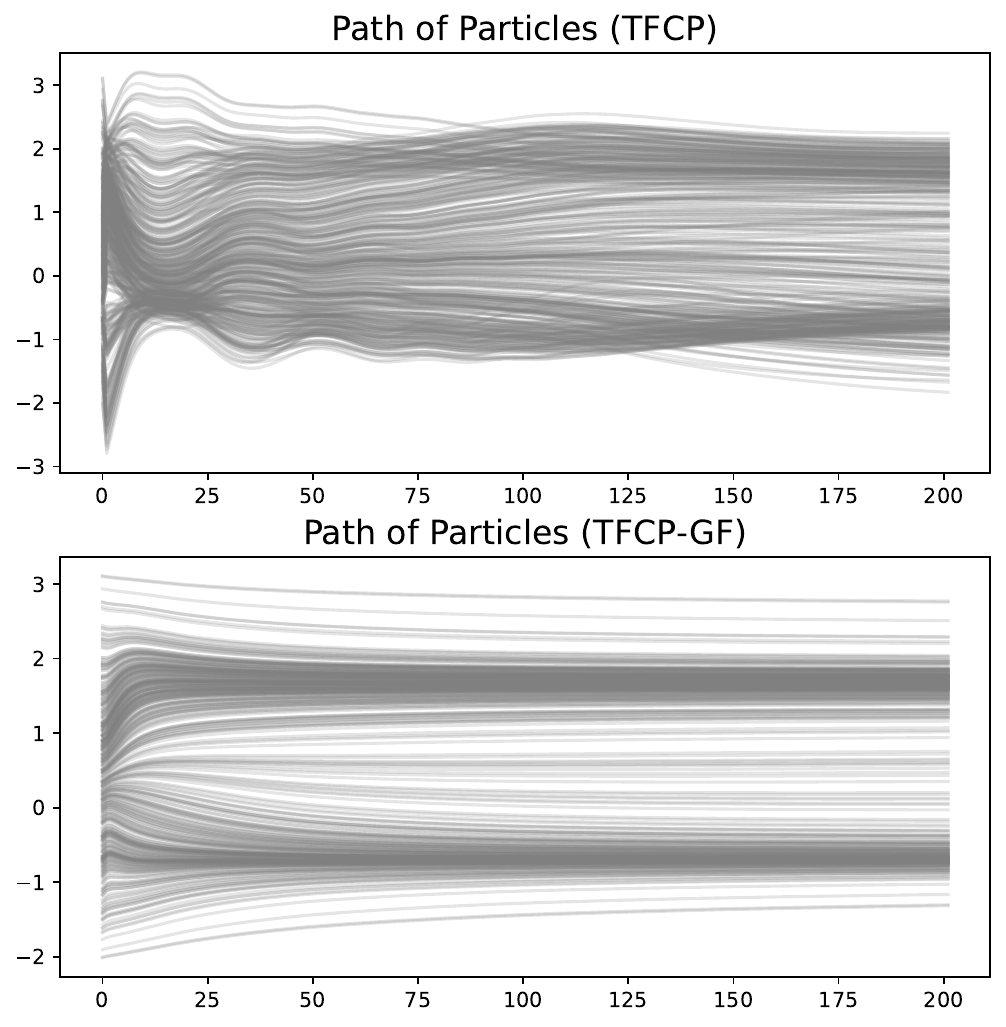}
			\end{minipage}
		}
		\caption{Left panel: Approximated posterior $P(X\mid Y^o)$ given by TFCP, TFCP-GF, EnKF and SIRPF. Middle panel: Scatter plot of joint measure $\pi_{\bx, \by}$ and $\mathcal{T}_{\sharp}(\pi_{\bx}\otimes \pi_{\by})$. Right Panel: Trajectories of particles transported from the prior to the approximate posterior.}
		\label{fig:1d-static}
	\end{figure}
	
	The left panel of \Cref{fig:1d-static} shows the approximated posterior $P(X\mid Y^o)$ generated from TFCP, TFCP-GF, EnKF and PF. All methods use a sample size of $N = 400$. In TFCP, we parameterize the transport map using three simple shallow fully connected neural network in \Cref{fig:transport-structure}. Compared with the reference posterior, both TFCP and TFCP-GF accurately estimated the bimodal distribution, whereas EnKF fails in this case. TFCP-GF even exhibits the same accuracy as PF. 
	Since TFCP and TFCP-GF, like the particle filter (PF), directly approximate the posterior distribution, they yield accurate posterior estimates when the objective function is sufficiently optimized. 
	The middle panel of \Cref{fig:1d-static} presents the scatter plot of the joint distribution $\pi_{\bx, \by}$ and $\mathcal{T}_{\sharp}(\pi_{\bx}\otimes \pi_{\by})$. Both TFCP and TFCP-GF accurately capture the dominant part of the joint distribution, although a small number of sample points are not effectively transported. However, since our approach is primarily based on matching expectations, this does not compromise the accuracy of the posterior estimation provided by the transport map.
	As in \Cref{alg:1} and \Cref{alg:2}, the proposed approachs involves first estimating the joint distribution, and subsequently extracting the conditional distribution corresponding to a specific value of $Y^o = 1.2$. 
	The right panel of \Cref{fig:1d-static} shows the particle trajectories for TFCP and TFCP-GF. Both methods gradually transport particles from the Gaussian prior toward the bimodal posterior. However, the trajectories of TFCP appear less structured and exhibit slower convergence as the number of iterations increases, whereas TFCP-GF produces smoother particle paths and achieves convergence more rapidly. Consistent with this observation, the left panel confirms that TFCP-GF yields a more accurate approximation of the posterior distribution.
	
	To better illustrate the differences between TFCP and PF, we present an example in a two-dimensional state space. For slightly higher dimensions, the PF may exhibit particle collapse. Suppose the prior of $\bX\in\mathbb{R}^2$ is a Gaussian distribution $\mathbf{N}(\boldsymbol{\mu}_0, \boldsymbol{\Sigma}_0)$, and the observation function is defined as 
	\begin{equation*}
		H(\bX) = \bX\odot \bX = X_1^2 + X_2^2,
	\end{equation*}
	where $\bX=[X_1, X_2]^{\top}$. Given $\boldsymbol{\mu}_0 = [0.5, 0.5]^{\top}$ and $\boldsymbol{\Sigma}_0=\mathbf{I}_{2\times 2}$, the posterior approximated by TFCP, EnKF and PF are shown in \Cref{fig:2d-static} when observation $Y^o = 1.5$ and observation noise is Gaussian $N(0, 0.5^2)$. To clearly demonstrate the approximation effects of different methods on the posterior distribution, we use 400 sampling particles and plotted a two-dimensional histogram of the approximate posterior. In this case, TFCP is still parameterized using a simple fully connected neural network.
	
	\begin{figure}[!htbp]
		\centering
		\subfigure{
			\begin{minipage}[t]{0.22\textwidth}
				\centering
				\includegraphics[width=\textwidth]{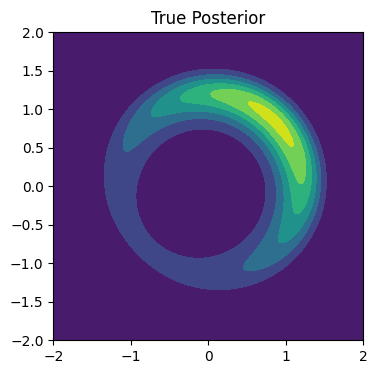}
			\end{minipage}
		}
		\subfigure{
			\begin{minipage}[t]{0.22\textwidth}
				\centering
				\includegraphics[width=\textwidth]{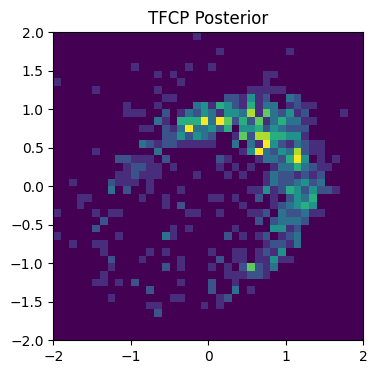}
			\end{minipage}
		}
		\subfigure{
			\begin{minipage}[t]{0.22\textwidth}
				\centering
				\includegraphics[width=\textwidth]{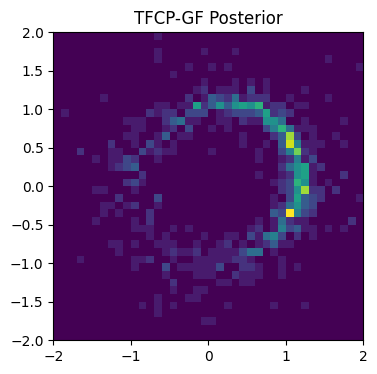}
			\end{minipage}
		}
		\subfigure{
			\begin{minipage}[t]{0.22\textwidth}
				\centering
				\includegraphics[width=\textwidth]{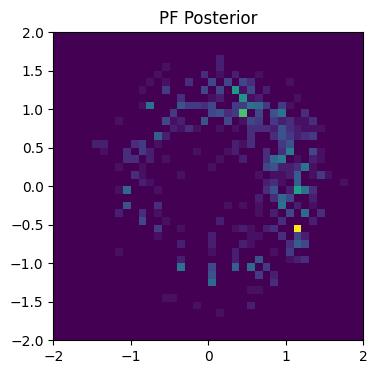}
			\end{minipage}
		}
		\caption{Approximated posterior $P(X\mid Y^o=1.5)$ given by TFCP, TFCP-GF and PF.}
		\label{fig:2d-static}
	\end{figure}
	
	As shown in \Cref{fig:2d-static},  TFCP, TFCP-GF and PF successfully capture the form of the non-Gaussian posterior. Among the three methods, TFCP-GF achieves the best approximation of the posterior distribution. TFCP may suffer from insufficient optimization, leading to poor estimation in low-probability regions, while PF collapses in a subset of sample points, exhibiting mild particle collapse rather than adequately covering the high-probability regions. 
	
	\subsection{Bene\v{s} Filter}\label{sec:4.2}
	
	In this subsection, we evaluate the effectiveness of TFCP in tracking the evolution of the posterior distribution over time in a dynamic system. Consider the following continuous time filter problem, 
	\begin{equation*}
		\begin{aligned}
			&\mathrm{d}X_t = \mu\sigma_B\tanh(\frac{\mu}{\sigma_B}X_t)\mathrm{d}t + \sigma_B \mathrm{d}B_t, X_0 = x_0, \\
			&\mathrm{d}Z_t = (h_1X_t + h_1h_2)\mathrm{d}t + \mathrm{d}W_t,
		\end{aligned}
	\end{equation*}
	where $\{B_t\}$ and $\{W_t\}$ are one-dimensional independent Brownian motion, $x_0$ is the initial state and the parameters $\mu, \sigma_B, h_1, h_2\in \mathbb{R}$. This problem, known as the Bene\v{s} filter, is one of the rare cases in continuous-time nonlinear filtering for which an exact closed-form solution can be derived. The analytical posterior to the filtering problem can be expressed by a mixture of two Gaussian distributions \cite{taghvaeiDiffusionMapbasedAlgorithm2020},
	\begin{equation*}
		\omega_tN(a_t-b_t, \sigma_t^2) + (1-\omega_t)N(a_t+b_t, \sigma^2_t),
	\end{equation*}
	where
	\begin{equation*}
		\begin{aligned}
			&a_t = \sigma_B\Psi_t\tanh(h_1\sigma_Bt) + \frac{h_2+x_0}{\cosh(h_1\sigma_Bt)}-h_2,\quad b_t = \frac{\mu}{h_1}\tanh(h_1\sigma_Bt),\\
			&\sigma_t^2 = \frac{\sigma_B}{h_1}\tanh(h_1\sigma_Bt),\quad
			\Psi_t = \int_0^t \frac{\sinh(h_1\sigma_Bs)}{\sinh(h_1\sigma_Bt)}\mathrm{d}Z_t,\quad
			\omega_t = \frac{1}{1+e^{\frac{2a_tb_t}{\sigma_B}\coth(h_1\sigma_Bt)}}.
		\end{aligned}
	\end{equation*}
	The filtering problem is parameterized by $\mu=0.5, \sigma_B=0.8, h_1=0.4, h_2=0$, and numerical examples are conducted over the time horizon $[0,3]$ with a temporal discretization step of $\Delta t= 0.1$. The stochastic ordinary differential equations are integrated numerically using the fourth-order Runge–Kutta scheme. Here, the TFCP is parametrized as the nudging-like form in \eqref{eq:3.7}. We define the mapping $\mathbf{T}^{\theta}$ using two fully connected layers augmented with two residual connection blocks, each layer containing 20 neurons. In each assimilation window, we train for 400 optimization steps.
	
	\Cref{fig:rmse-benes} presents a comparison of SIRPF, TFCP, and TFCP-GF in terms of estimating the true posterior mean, with all methods utilizing $400$ particles. Among the three methods, TFCP and TFCP-GF yield more accurate approximations of the posterior mean compared to SIRPF. 
	While all three methods aim to approximate the posterior distribution, the SIRPF tends to accumulate additional stochastic error introduced by the resampling step, which propagates and amplifies over time. In contrast, the proposed TFCP and TFCP-GF, which are based on deterministic transport maps, exhibit greater robustness.
	In this experiment, the performance of TFCP and TFCP-GF is comparable, as the true posterior distribution is nearly unimodal, allowing TFCP to sufficiently optimize the target loss function.
	
	\begin{figure}[!htbp]
		\centering
		\includegraphics[width=.45\textwidth]{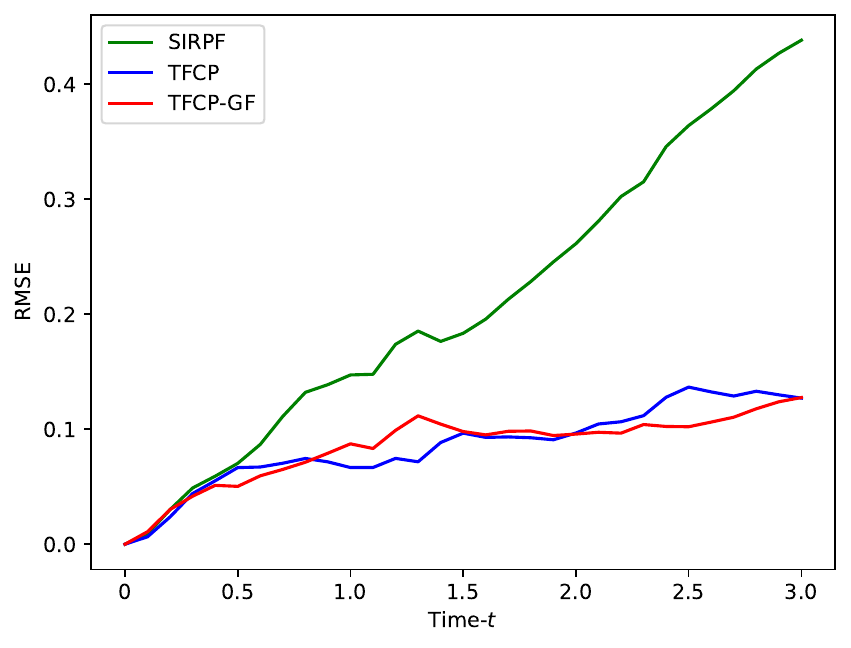}
		\caption{Plot of RMSE in estimating the posterior mean.}
		\label{fig:rmse-benes}
	\end{figure}
	
	To illustrate the accuracy of TFCP in estimating the posterior distribution, we select three time points: $t=1.0$, $t=2.0$, and $t=3.0$. \Cref{fig:dist-benes} presents histograms of the particle distributions obtained from SIRPF, TFCP and TFCP-GF filtering methods at these time points, highlighting their performances in posterior distribution estimation. At the time $t=1.0$, the posterior estimates from three methods are relatively close to the true posterior, with TFCP-GF's histogram being the closest to the true posterior. However, as time evolves, the posterior obtained by SIRPF becomes far away from the true posterior distribution, exhibiting substantial bias. Although TFCP and TFCP-GF also show some degradation in performance, they continue to provide overall accurate approximations of the posterior distribution.
	\begin{figure}[!h]
		\centering
		\subfigure{
			\begin{minipage}[t]{0.7\textwidth}
				\centering
				\includegraphics[width=\textwidth]{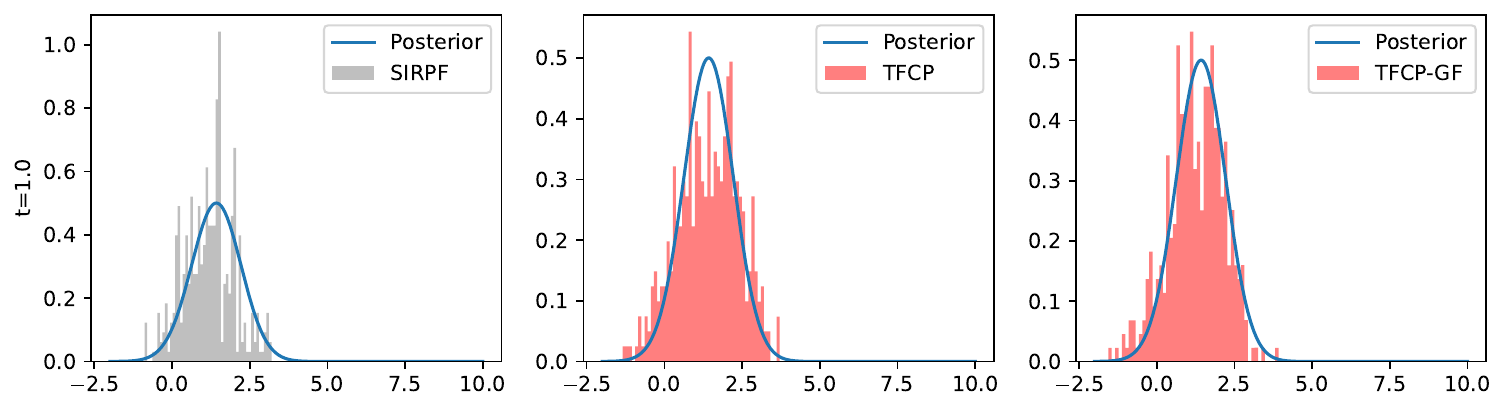}
			\end{minipage}
		}\\
		\subfigure{
			\begin{minipage}[t]{0.7\textwidth}
				\centering
				\includegraphics[width=\textwidth]{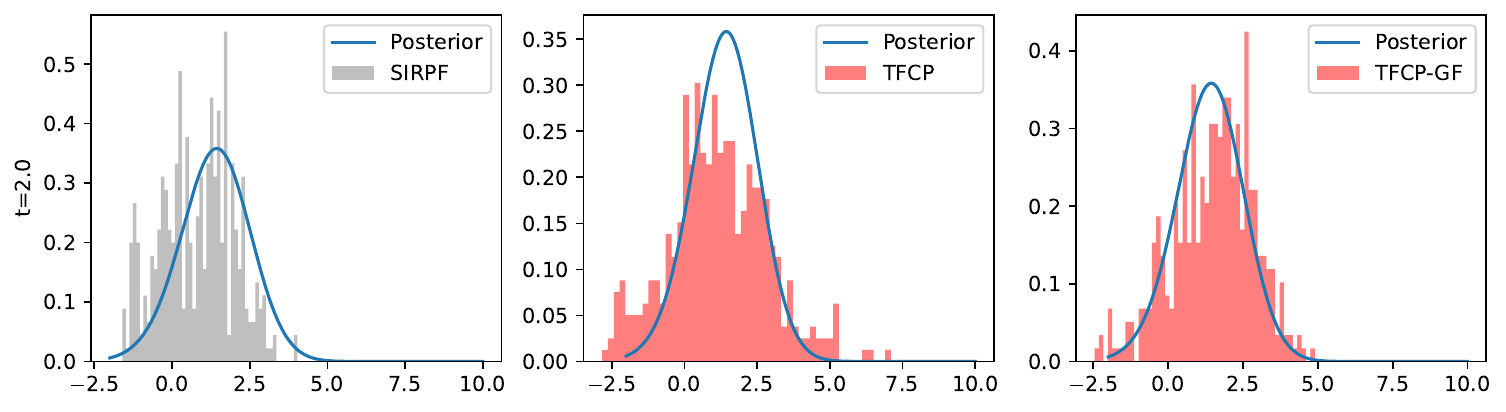}
			\end{minipage}
		}\\
		\subfigure{
			\begin{minipage}[t]{0.7\textwidth}
				\centering
				\includegraphics[width=\textwidth]{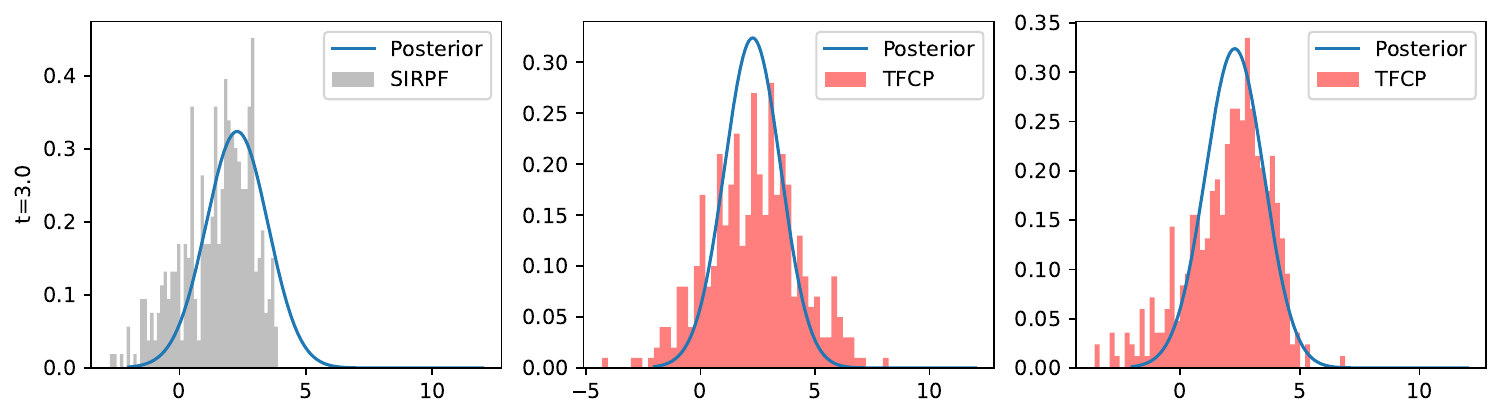}
			\end{minipage}
		}
		\caption{The histograms of particles obtained using PF, TFCP and TFCP-GF as compared with the exact distribution at time $t=1.0$, $t=2.0$, and $t=3.0$.}
		\label{fig:dist-benes}
	\end{figure}
	
	\subsection{Lorenz'63 System}\label{sec:4.3}
	
	In this subsection, we assess the performance of TFCP through its application to state estimation in a moderately complex Lorenz'63 system. It describes the nonlinear dynamical behavior in atmospheric circulation. The Lorenz'63 model is defined by the following three coupled nonlinear ordinary differential equations:
	\begin{equation*}
		\begin{aligned}
			{\mathrm{d}X_1} &= -\sigma (X_1 + X_2){\mathrm{d}t} + \gamma_1 {\mathrm{d} B_1},\\
			{\mathrm{d}X_2} &= (-X_1X_2+\rho X_1 - X_2){\mathrm{d}t} + \gamma_2 {\mathrm{d} B_2},\\
			{\mathrm{d}X_3} &= (X_1X_2-\beta X_3){\mathrm{d}t} + \gamma_3 {\mathrm{d} B_3}. 
		\end{aligned}
	\end{equation*}
	Here $\{X_1\}$, $\{X_2\}$ and $\{X_3\}$ are one dimensional stochastical process, $\{B_1\}$, $\{B_2\}$ and $\{B_3\}$ are brownian motions assumed to be independent. The parameters $\sigma=10$, $\beta=8/3$ and $\rho=28$ lead to the well-known Lorenz attractor. The noise level of different state variables are set to be $\gamma_1=\gamma_2=\gamma_3=4.0\times 10^{-4}$, and then we simulated the continuous time dynamical system using fourth order Runge-Kutta method with a time stepsize $\Delta t=0.01$. The initial condition of the dynamical system is set by a Gaussian distribution $\mathcal{N}(\mathbf{m}_0, \mathbf{C}_0)$, where $m_0$ is zero vector and $\mathbf{C}_0=\mathbf{I}_{3\times 3}$. Partial observation operator is considered, 
	\begin{equation*}
		H(\bX) = X_1 + \varepsilon,
	\end{equation*}
	where $\bX = [X_1, X_2, X_2]^{\top}$ and $\varepsilon\sim\mathcal{N}(0, \eta^2)$ is the additive Gaussian noise, we take $\eta=1.0$. 
	Partial observation introduces significant complexity and severe nonlinearities, posing substantial challenges for accurate state estimation and posterior approximation. Therefore, we compare our method with the widely used method in nonlinear filtering, the EnKF.
	
	As in the previous subsection, the transport map $\mathbf{T}^{\theta}$ defined in \eqref{eq:3.7} is parameterized by two fully connected layers augmented with two residual connection blocks, each layer containing 20 neurons. The loss function \eqref{eq:3.9} is optimized using the AdamW algorithm with a learning rate of $10^{-2}$ and a fixed number of $200$ training iterations. \Cref{fig:rmse-Lorenz63} presents the average RMSE performance of both TFCP, TFCP-GF and EnKF across various observation intervals and ensemble sizes.
	\begin{figure}[!h]
		\centering
		\subfigure{
			\begin{minipage}[t]{0.43\textwidth}
				\centering
				\includegraphics[width=\textwidth]{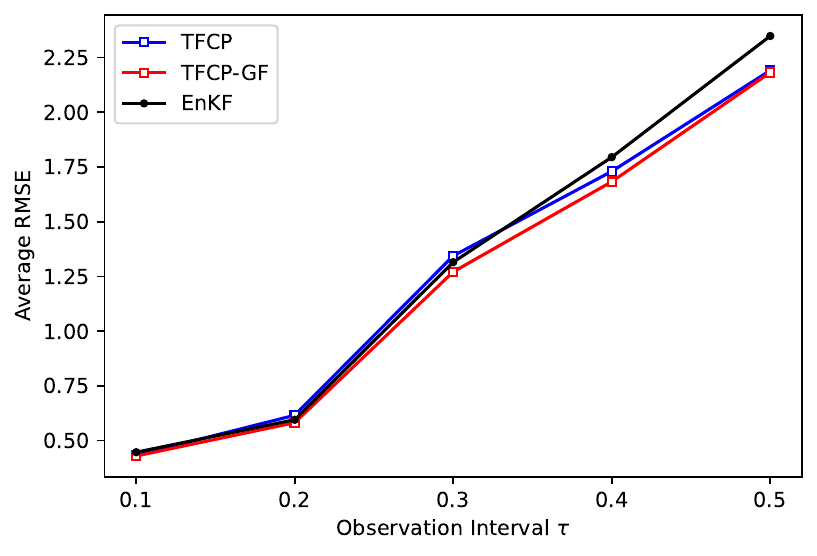}
			\end{minipage}
		}
		\subfigure{
			\begin{minipage}[t]{0.43\textwidth}
				\centering
				\includegraphics[width=\textwidth]{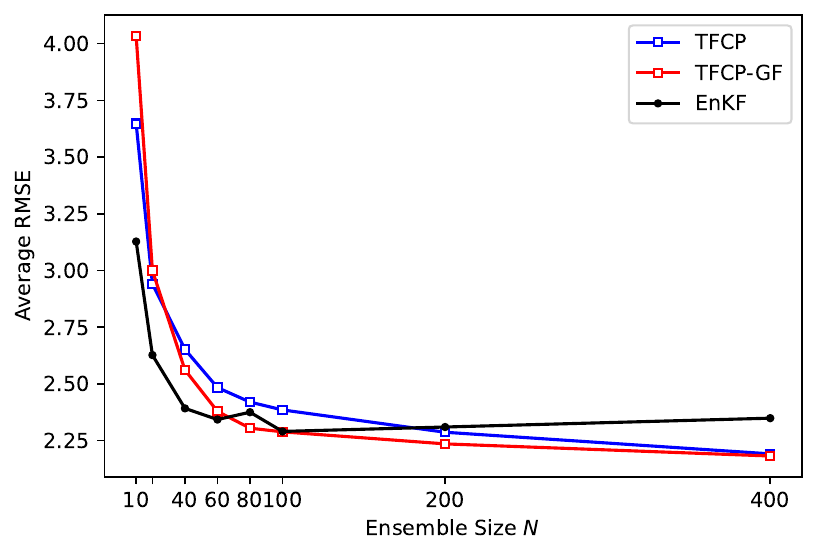}
			\end{minipage}
		}
		\caption{Average RMSE of Lorenz'63 system for ensemble size $N=400$ as a function of observation interval (Left), and for observation interval $\Delta t_{ob}=0.5$ as a function of ensemble size (Right).}
		\label{fig:rmse-Lorenz63}
	\end{figure}
	
	Left panel of the \Cref{fig:rmse-Lorenz63} presents the average RMSE of Lorenz'63 system for ensemble size $N=400$ for different observation intervals. In scenarios with small observation intervals $\tau \in \{0.1, 0.2, 0.3\}$, the state estimation performance of TFCP and TFCP-GF is close to EnKF. As the observation interval increases $\tau \in \{0.4, 0.5\}$, TFCP and TFCP-GF achieve higher estimation accuracy. With increasing observation intervals, the filtering problem exhibits stronger nonlinear and non-Gaussian characteristics. This leads to a degradation in the performance of EnKF, which relies on Gaussian approximations. Right panel of \Cref{fig:rmse-Lorenz63} shows the average RMSE under a fixed observation interval $\tau=0.5$ for varying ensemble sizes. The accuancy of TFCP-GF improves with larger ensemble sizes $N\ge 60$ compare to EnKF, while TFCP requires ensemble size large than $N=200$. TFCP-GF demonstrates superior performance and stability compared to TFCP.
	
	\Cref{fig:cp-lorenz63} presents a comparison of TFCP, TFCP-GF and EnKF in terms of uncertainty quantification for the filtering problem, evaluated using the Coverage Probability. For the $95\%$ confidence interval, TFCP and TFCP-GF achieve a Coverage Probability value that is closer to the nominal $95\%$ coverage, indicating more accurate and reliable uncertainty estimation compared to EnKF.
	
	\begin{figure}[!htbp]
		\centering
		\includegraphics[width=.45\textwidth]{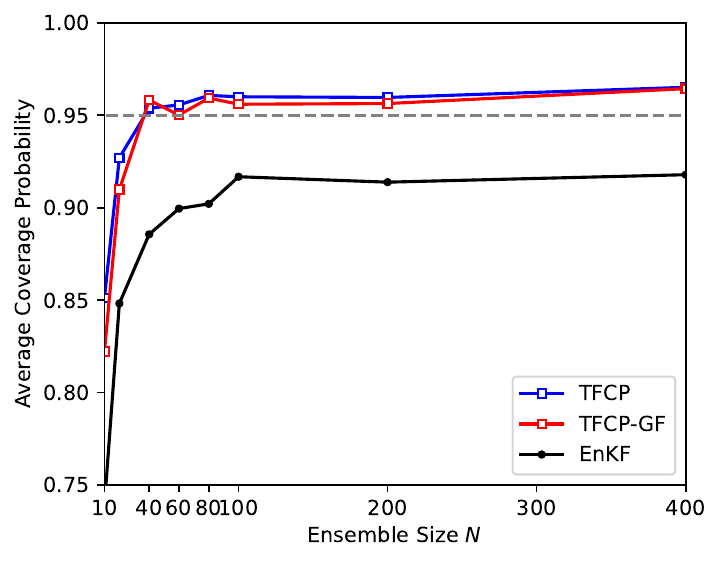}
		\caption{Plot of Coverage Probability of Lorenz'63 system for observation interval $\Delta t_{ob}=0.5$ as a function of ensemble size.}
		\label{fig:cp-lorenz63}
	\end{figure}
	
	When only the third component of the Lorenz ’63 system is observed, the marginal posterior distributions of the first two components are bimodal. Here, we present the distribution estimation results of the Lorenz ’63 system obtained by TFCP.
	
	\begin{figure}[!htbp]
		\centering
		\includegraphics[width=.65\textwidth]{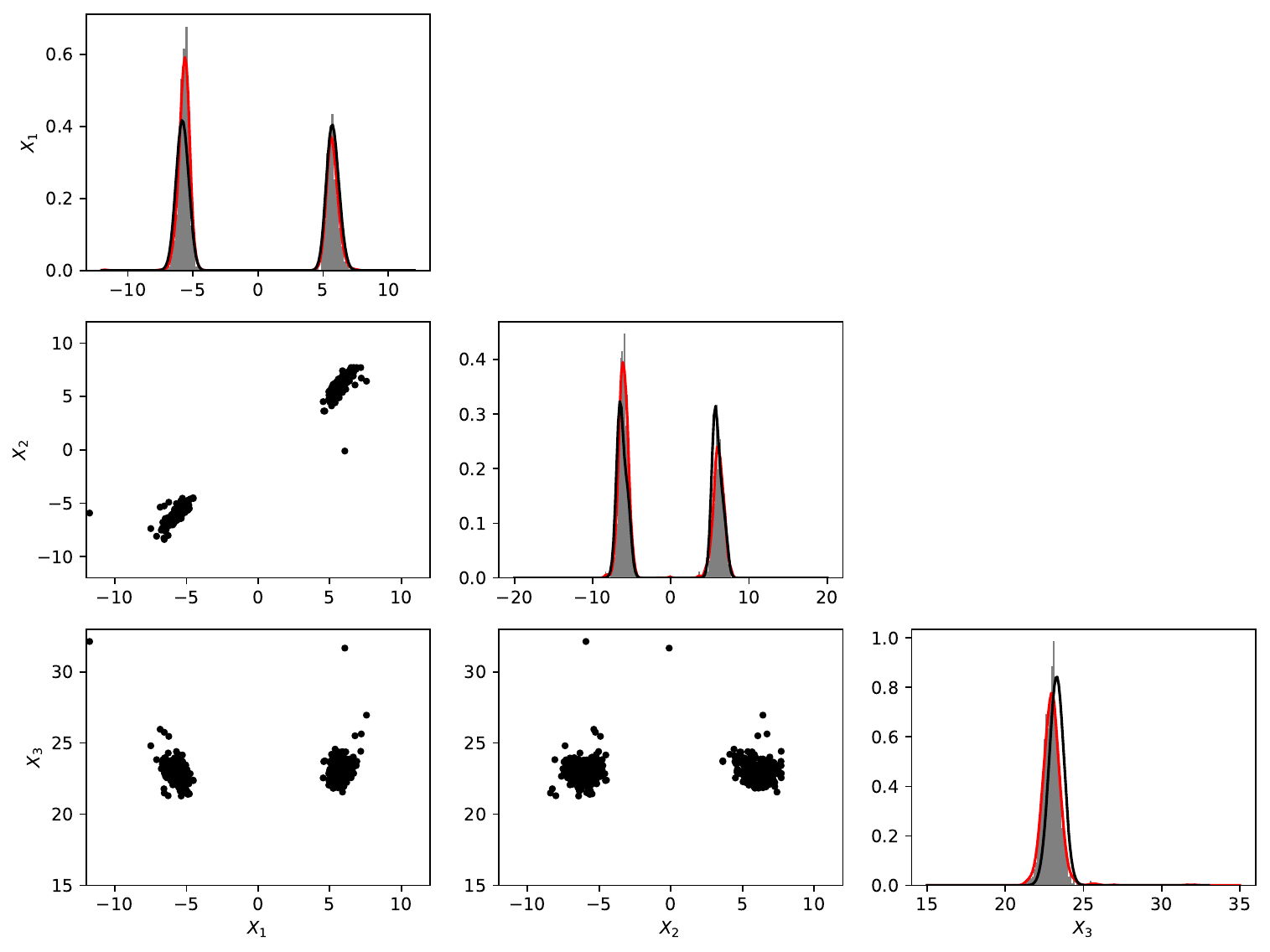}
		\caption{Scatter and plot of the marginal posterior distribution approximated by TFCP.}
		\label{fig:cp-lorenz63_X3}
	\end{figure}
	
	In Figure 4.10, the black curves represent the posterior distribution approximated by 10,000 particles using particle filtering. It can be seen that when only the third dimension is observed, the posterior distributions of the first two components are bimodal, while the marginal posteriors for all three dimensions appear approximately Gaussian. TFCP accurately captures the multimodal structure of the posterior distribution. Moreover, compared to standard particle filtering, TFCP prevents particle collapse by preserving particle diversity through the transport map, enabling a more accurate posterior approximation even with a significantly smaller number of particles.
	
	\subsection{Lorenz'96 System}\label{sec:4.4}
	
	In this subsection, we consider a high-dimensional dynamical system. In this example, we mainly evaluate the performance of TFCP in state estimation for high-dimensional problems, as well as the effectiveness of the localization. The Lorenz'96 system is a simplified mathematical model designed to capture the nonlinear dynamical behavior of atmospheric circulation systems. The state at time $t$ is a $J$-dimensional vector $\mathbf{X}(t)=\left[X_1(t),X_2(t), \cdots X_J(t)\right]^{\top}$, which is governed by the ODE system
	\begin{equation*}
		\begin{aligned}
			\frac{\mathrm{d}X_i}{\mathrm{d}t} &= (X_{i+1}-X_{i-2})X_{i-1} - X_i + F + \gamma\frac{\mathrm{d} B_i}{\mathrm{d}t}, i\in\{1,\cdots, J\},\\
			X_0&=X_J, X_{J+1}=X_1, X_{-1} = X_{J-1},
		\end{aligned}
	\end{equation*}
	with a constant forcing term $F$ and independent Brownian motions $\{B_i\}$. In this example, we take $F=8$, which leads to a fully chaotic dynamic. The noise level is set to be $\gamma=4\times 10^{-4}$, and a fourth-order Runge-Kutta method is used to integrate the ODEs with a time stepsize $\Delta t= 0.01$. The hard case setup for Lorenz'96 system, with sparse observation in both space and time, is considered. The dimension of state and observation are $J=40$ and $d=20$ (observing every other component of the state), repectively. Each observation is corrupted by an independent additive Gaussian noise with variance $\eta^2 = 1.0$. 
	
	The number of particles grows exponentially with respect to the dimension of state space to  accurately approximate the posterior distribution, which is computationally prohibitive.
	To address this challenge, we adopt the TFCP with localiaztion. And the Local Ensemble Transform Kalman Filter (LETKF) is selected as the baseline method for performance comparison. The use of localization enables parallel updating of different local blocks. The transport map for each local block is constructed by following the same configuration as in the Lorenz'63 example. \Cref{fig:rmse-Lorenz96} shows the average RMSE and average coverage probability of TFCP and LETKF under an observation interval of $\Delta t_{ob} = 0.5$ and a localization radius of $r_{loc}=1$. With a localization radius of $1$, each local block comprises three state components and is associated with one to two observations.
	
	\begin{figure}[!h]
		\centering
		\subfigure{
			\begin{minipage}[t]{0.43\textwidth}
				\centering
				\includegraphics[width=\textwidth]{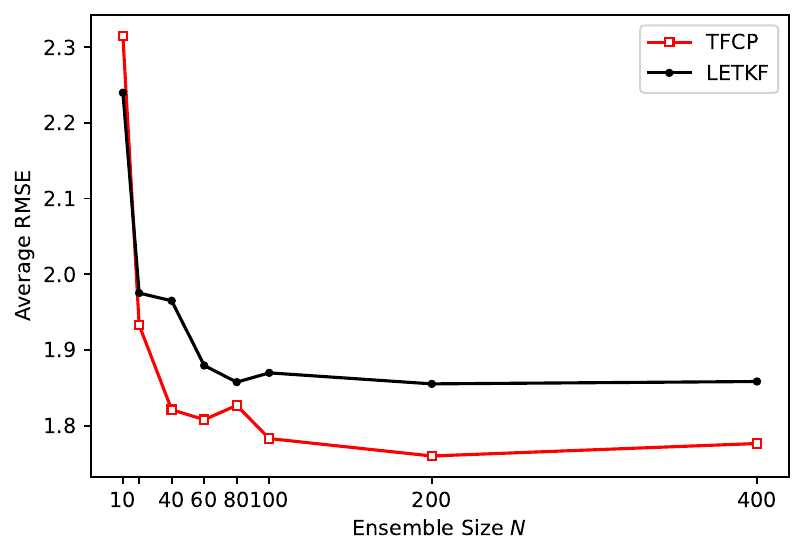}
			\end{minipage}
		}
		\subfigure{
			\begin{minipage}[t]{0.44\textwidth}
				\centering
				\includegraphics[width=\textwidth]{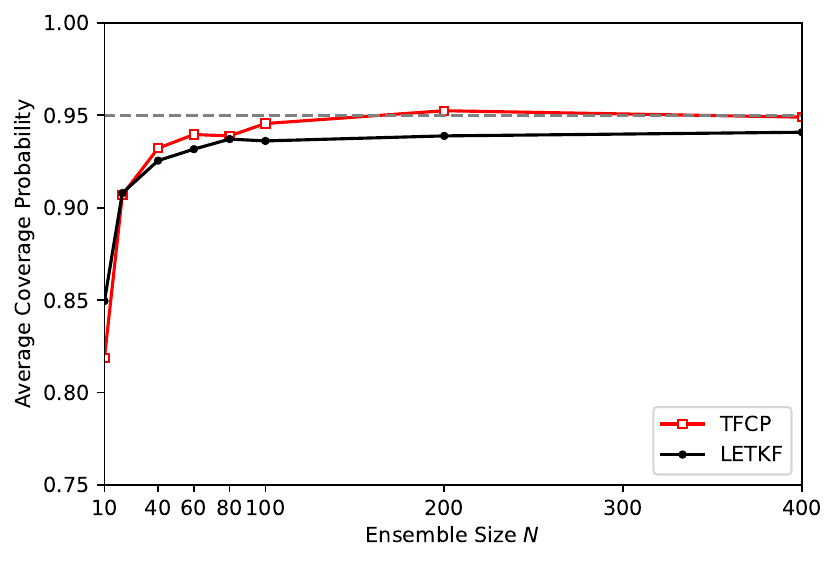}
			\end{minipage}
		}
		\caption{Average RMSE (Left) and Average Coverage Probability (Right) of Lorenz'96 system for observation interval $\Delta t_{ob}=0.5$ as a function of ensemble size.}
		\label{fig:rmse-Lorenz96}
	\end{figure}
	
	Left panel of \Cref{fig:rmse-Lorenz96} presents the RMSE under different ensemble sizes. Compare with LETKF, TFCP demonstrates better accuracy in estimating the true state across various ensemble sizes. Both TFCP and LETKF reach their optimal accuracy when the ensemble size reaches $100$. However, TFCP achieves higher estimation accuracy than LETKF while requiring a much smaller number of particles. TFCP achieves accuracy beyond the optimal level of LETKF with only $40$ particles. TFCP guides the ensemble toward a more accurate approximation of the posterior distribution, enabling it to retain more information of history in prior over time evolution.
	Right panel of \Cref{fig:rmse-Lorenz96} shows the average coverage probability of $95\%$ confidence interval under different ensemble sizes. By the numerical results, TFCP demonstrates better agreement with the theoretical $95\%$ coverage level than LETKF. This shows that TFCP produces posterior estimates with better uncertainty quantification.
	
	\begin{figure}[!htbp]
		\centering
		\includegraphics[width=.45\textwidth]{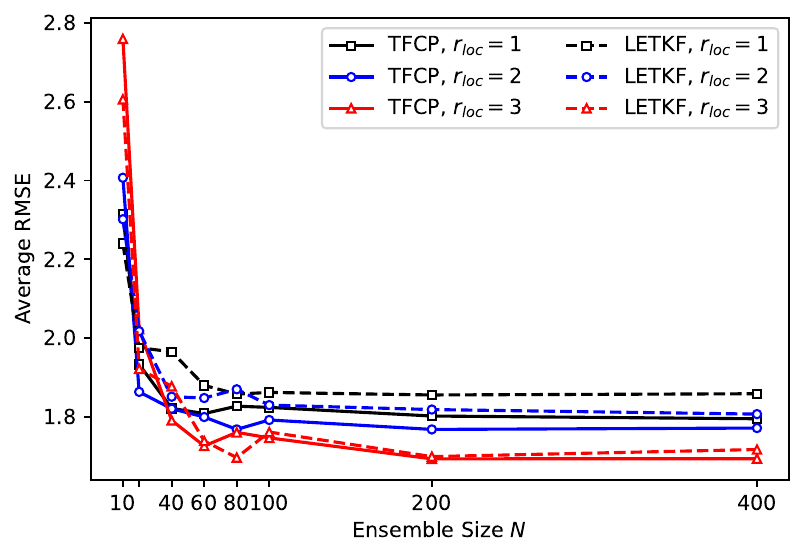}
		\caption{Plot of Average RMSE of Lorenz'96 system for observation interval $\Delta t_{ob}=0.5$ and different domain localization radius $r_{loc}$ as a function of ensemble size.}
		\label{fig:rmse-l96-loc}
	\end{figure}
	
	\Cref{fig:rmse-l96-loc} shows the comparison of average RMSE between TFCP and LETKF under different localization radius. Both TFCP and LETKF exhibit improvements in average RMSE when using larger localization radius. This can be attributed to two factors: (1) a larger localization radius incorporates more observational data into each local update; (2) it mitigates the errors introduced by truncating distant correlations, leading to more accurate state estimates. It is observed that the advantage of TFCP over LETKF become small with increasing localization radius. This can be attributed to the fact that, under a fixed ensemble size, approximating the posterior distribution of a high dimensional state space becomes significantly challenging. 
	
	\subsection{Kolmogorov flow}\label{sec:4.5}
	
	In this section, we consider the assimilation of a Kolmogorov flow, which is a viscous and incompressible fluid flow defined on the two-dimensional periodic domain $[0, 2\pi]^2$ with a single-mode sinusoidal external forcing,
	\begin{equation}
		\begin{cases}
			\frac{\partial \mathbf{u}}{\partial t} = -(\mathbf{u}\cdot \nabla)\mathbf{u}+\nu\nabla^2\mathbf{u}-\frac{1}{\rho}\nabla p+\mathbf{F},\\
			\nabla\cdot \mathbf{u}=0,
		\end{cases}
	\end{equation}
	where $\mathbf{u}$ is the velocity field, $\nu$ is the kinematic viscosity, $p$ is the pressure and $\rho$ is the fluid density. In this example, we take $\nu=10^{-3}$ and $\rho=1$.The {\it jax-cfd} package \cite{kochkovMachineLearningAccelerated2021} is used to simulate the model with the time discretization step $\Delta t=0.001$. 
	In contrast to methods such as the EnKF, the proposed approach, based on transport maps, does not require vectorizing the two-dimensional field and instead directly constructs a field-to-field transformation. We employ a CNN to implement TFCP, thereby reducing the number of parameters in the transport map. In this setting, the ensemble size is $N=400$. 
	Here, we consider the data assimilation problem under two scenarios: sparse observations and partial observations, respectively. Figures \ref{fig:Kol_flow_saprse} and \Cref{fig:Kol_flow_partial} show the filtered posterior state estimates and the corresponding absolute error at five distinct time instants obtained by TFCP.
	In both cases, TFCP effectively integrates model dynamics and observational information to yield accurate state estimates. The TFCP is capable of reconstructing the unobserved components of the velocity field. As time evolves and more observations become available, the prior field, incorporating historical information, enables increasingly accurate estimation of the velocity field.
	
	\begin{figure}[!ht]
		\centering
		\includegraphics[width=.85\textwidth]{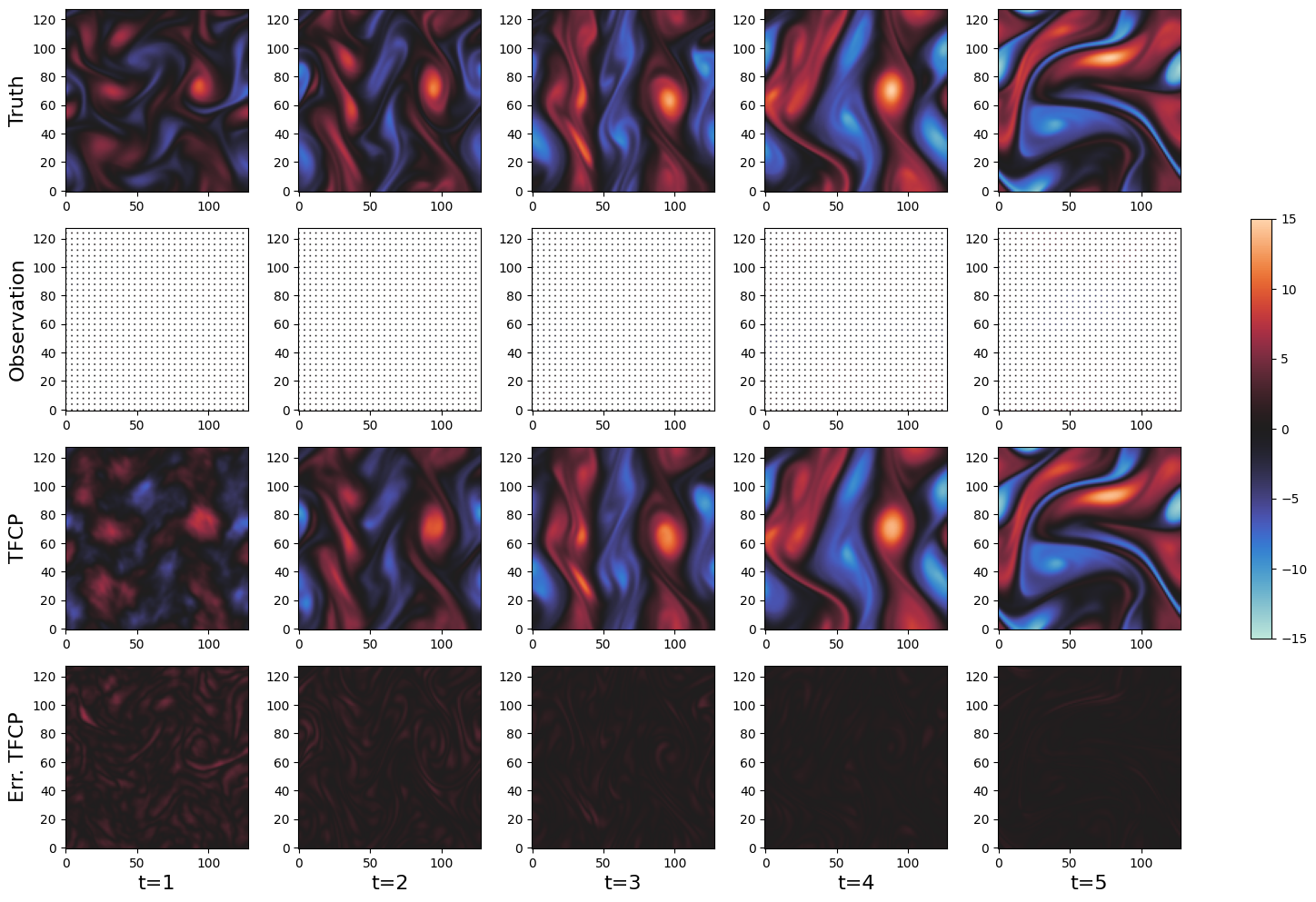}
		\caption{Filtering result of Vorticity of Kolmogorov flow, $1/16$ of the states are observed.}
		\label{fig:Kol_flow_saprse}
	\end{figure}
	
	\begin{figure}[!ht]
		\centering
		\includegraphics[width=.85\textwidth]{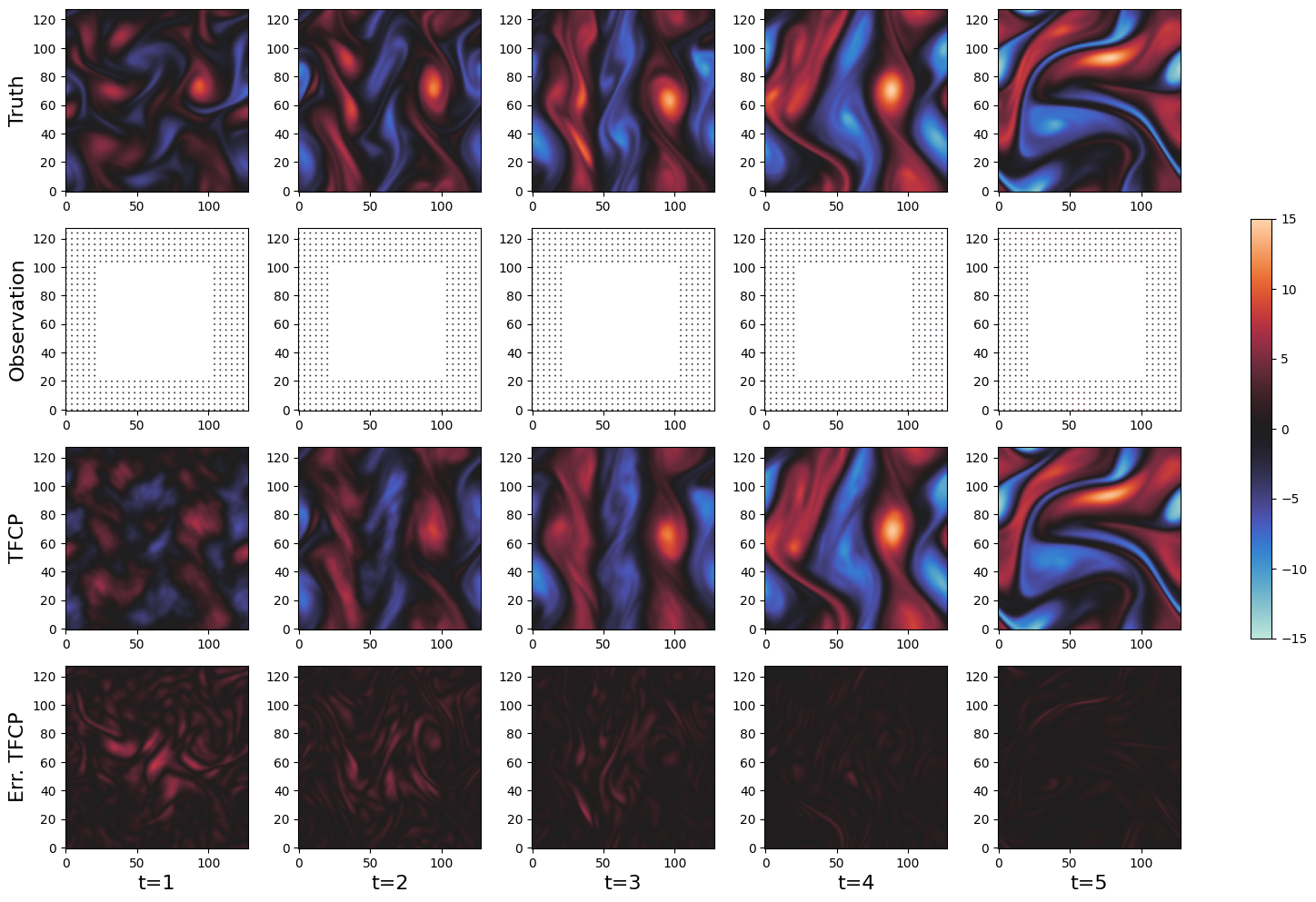}
		\caption{Filtering result of Vorticity of Kolmogorov flow, partial states are observed.}
		\label{fig:Kol_flow_partial}
	\end{figure}
	
	\section{Conclusion}\label{sec:5}
	In this work, we have introduced a likelihood-free transport filtering method that leverages couplings between state and observation variables through a block-triangular transport map. By formulating the filtering analysis step as the minimization of the MMD between the true joint distribution and its transport-induced approximation, our approach avoids explicit likelihood evaluations and naturally accommodates nonlinear, non-Gaussian dynamics. A key innovation is the use of gradient flows in reproducing kernel Hilbert spaces to derive an analytically tractable sequence of steepest-descent transport maps without iterative training or parametric optimization while ensuring convergence toward the true posterior. We established a covergence analysis, showing that the expected MMD between the approximate and exact posteriors decays at a rate of $\mathcal{O}(1/\sqrt{N})$. Furthermore, we extended the method to high-dimensional settings via a domain-localization strategy, enabling scalable and parallelizable inference. Numerical examples confirmed that the proposed filter outperforms conventional approaches in challenging nonlinear and non-Gaussian scenarios, offering both robustness against particle collapse and accurate posterior approximation. This framework opens a new avenue for likelihood-free Bayesian filtering grounded in optimal transport and kernel-based statistical divergences.
	
	\medskip
	\noindent\textbf{Acknowledgment:}
	L.~Jiang acknowledges the support of NSFC 12271408.
	
	\bibliographystyle{siamplain}
	\bibliography{references}
	
\end{document}